%% file: main.tex
\documentclass{article}
\newcommand{\PaperPreprint}{} 
\PassOptionsToPackage{table}{xcolor}
\usepackage{styles/fancyhdr,styles/natbib}
\makeatletter
\@namedef{ver@fancyhdr.sty}{}
\makeatother
\usepackage{styles/iclr2027_conference,times}
\input{styles/math_commands.tex}

\usepackage{amsthm}
\usepackage{amssymb}
\usepackage{algorithm}
\usepackage{algpseudocode}
\usepackage{booktabs}   
\usepackage{graphicx}   
\usepackage{subcaption}
\usepackage{enumitem}
\usepackage{makecell}
\usepackage{multirow}
\usepackage{xcolor}
\usepackage[most]{tcolorbox}
\usepackage{wrapfig}

\usepackage{hyperref}
\usepackage{url}
\usepackage[T1]{fontenc}
\tcbuselibrary{skins,breakable}

\makeatletter
\renewcommand\hyper@natlinkbreak[2]{#1}
\makeatother

\newtheorem{theorem}{Theorem}
\newtheorem{lemma}{Lemma}
\newtheorem{definition}{Definition}
\newtheorem{corollary}{Corollary}

\newtheorem{assumption}{Assumption}

\newcommand{\meanstd}[2]{%
  \ensuremath{#1_{\scriptscriptstyle\pm #2}}}
\newcommand{\bestmeanstd}[2]{%
  \ensuremath{\mathbf{#1}_{\scriptscriptstyle\pm #2}}}

\definecolor{PromptHeader}{HTML}{6C655E}
\definecolor{PromptText}{HTML}{5E5852}
\definecolor{PromptLine}{HTML}{C9D1D9}
\definecolor{PromptMuted}{HTML}{536170}

\newtcolorbox{promptbox}{
  enhanced,
  colback=PromptText!2,
  colframe=PromptLine,
  colbacktitle=PromptHeader,
  coltitle=white,
  boxrule=0.5pt,
  arc=1.5mm,
  left=8pt,
  right=8pt,
  top=6pt,
  bottom=6pt,
  fonttitle=\bfseries,
  title={Prompt template},
  before skip=7pt,
  after skip=7pt
}

\title{Learning Perturbation Robust Policies for LLM Agents with Stable Optimization}

\input{authors}

\input{paper_version}
\begin{document}

\maketitle
\begin{abstract}
Reinforcement learning (RL) has become an effective post-training paradigm for long-horizon large language model (LLM) agents. However, we find that the resulting policies can be sensitive to various policy perturbations, such as hidden-state noise, pruning, and quantization. In this work, we study how to improve perturbation robustness during policy optimization. We first introduce the notion of a perturbation robust policy and analyze conditions under which perturbed policy updates preserve stable monotonic improvement. Based on this analysis, we introduce Stable Perturbation-Robust Policy Optimization (SPrPO), which applies adaptive and sensitivity-aware perturbations during RL training. We evaluate SPrPO on ALFWorld and WebShop and conduct systematic experiments across multiple perturbation types and scales, showing improved perturbation robustness while maintaining stable policy optimization.
\end{abstract}

\input{sections/introduction}
\input{sections/related_work}
\input{sections/background}
\input{sections/theoretical_analysis}

\input{sections/method}

\input{sections/experiments}
\input{sections/conclusion}

\newpage

\input{statements/acknowledgment}

\input{statements/ai_use_statement}
\input{statements/reproducibility_statement}

\input{main.bbl}
\appendix
\newpage
\input{sections/appendix}

\end{document}

%% file: styles/math_commands.tex
\usepackage{amsmath,amsfonts,bm}

\def\eqref#1{equation~\ref{#1}}

\def\1{\bm{1}}

\DeclareMathAlphabet{\mathsfit}{\encodingdefault}{\sfdefault}{m}{sl}
\SetMathAlphabet{\mathsfit}{bold}{\encodingdefault}{\sfdefault}{bx}{n}



%% file: authors.tex
\author{%
Pengxin Wang, Yuanzhe LI, Yuxin Ren\\
\textbf{Huanrui Yang, Jingdi Chen}\\
Department of Electrical and Computer Engineering\\
University of Arizona
}

%% file: paper_version.tex
\ifdefined\PaperPreprint
  \iclrfinalcopy
  \let\PaperOriginalMakeTitle\maketitle
  \renewcommand{\maketitle}{%
    \PaperOriginalMakeTitle
    \lhead{Preprint}%
    \ifdefined\AuthorMetadataPending
      \begin{center}
        \small\bfseries Draft preprint: author metadata is incomplete.\\
        Complete authors.tex before distributing this version.
      \end{center}
    \fi
  }
\else
  \iclrfinalfalse
\fi

%% file: sections/introduction.tex
\section{Introduction}

RL has become an effective paradigm for post-training LLM agents on long-horizon decision making, where the model generates textual actions in interactive environments such as web navigation and embodied household tasks~\citep{agile,webrl,gigpo}. In practice, the deployed policy is not always the exact policy that was trained: serving costs push practitioners to prune or quantize the model, and inference may run under a different numerical precision. On static, single-turn tasks such as language modeling and question answering, where the model never acts on an environment, such perturbations are largely benign: pruning and post-training quantization compress a model substantially while retaining much of its performance~\citep{slicegpt,llmpruner,optq}. This tolerance does not carry over to RL-trained agents. As shown in Figure~\ref{fig:motivation}, pruning only $20\%$ of the feed-forward network (FFN) channels nearly collapses performance on WebShop~\citep{webshop}. This gap makes the robustness of RL-trained LLM agents to deployment-time perturbations an important concern for their practical use.

\begin{wrapfigure}[19]{r}{0.50\textwidth}
    \centering
    \includegraphics[width=\linewidth]{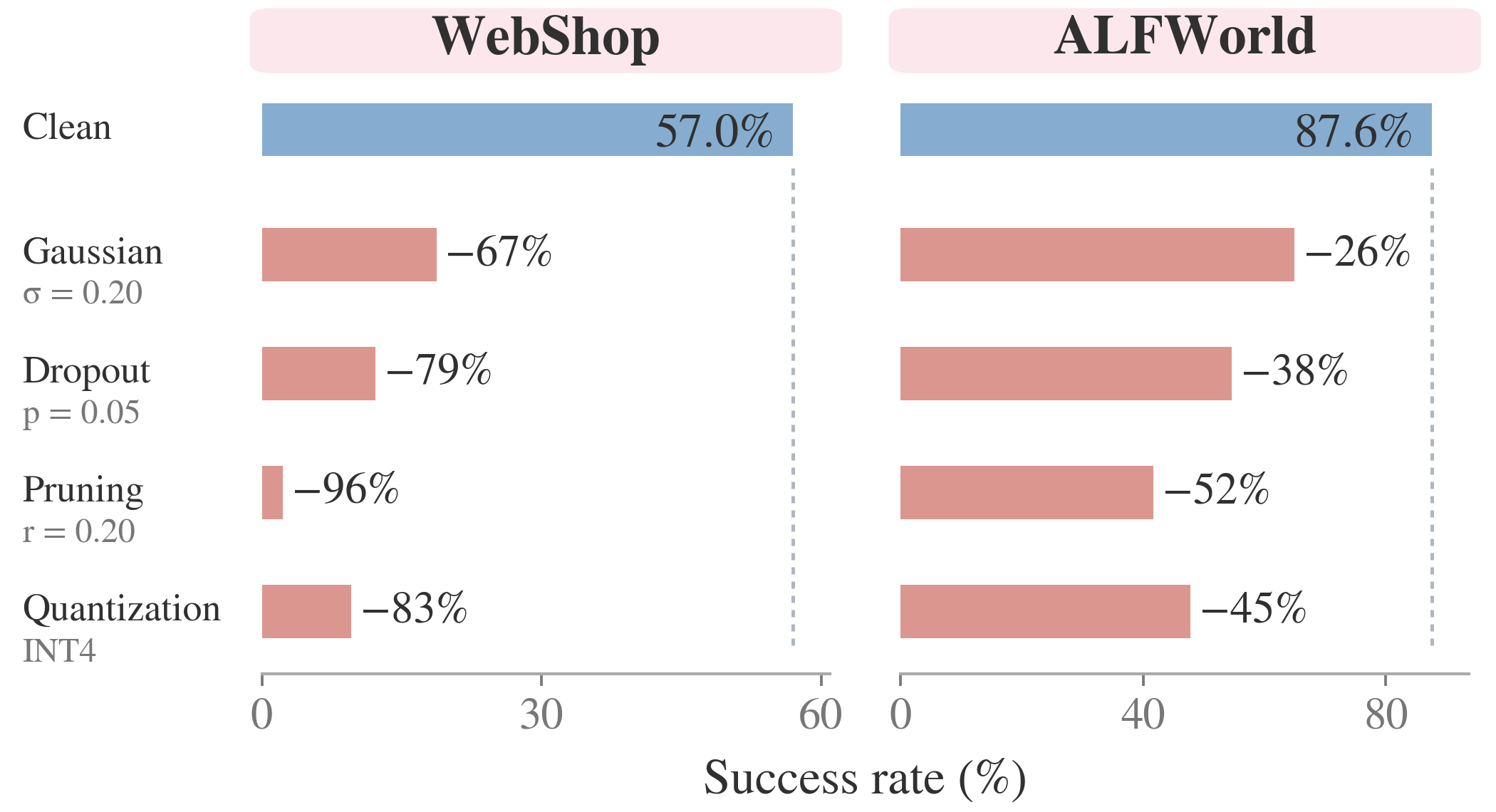}
    \caption{
        LLM agents trained by RL are sensitive to policy perturbations. Even with mild perturbation magnitude, success rates degrade substantially under these perturbations on both WebShop and ALFWorld. $\sigma$ denotes the standard deviation for Gaussian distribution, $p$ denotes the dropout probability and $r$ denotes pruning ratio.
    }
    \label{fig:motivation}
\end{wrapfigure}

A natural approach to improving such robustness is to inject perturbations into the policy during training, an idea widely studied in robust RL on control problems~\citep{RARL, ARRL, flat_reward}. Such methods typically operate on small policy networks trained from scratch, with dense rewards and low-dimensional continuous action spaces, so that injected perturbation has relatively mild effect. In the LLM setting, this idea has recently begun to be explored. NPFT~\citep{tamingweights} injects random noise into quantization-sensitive weights during fine-tuning, flattening the local loss landscape and mitigating the quantization error; related noise injection has also appeared in LLM RL, though there primarily as a source of exploration~\citep{qerl}. For long-horizon LLM agents, robustness to policy perturbation remains largely unstudied, and existing methods usually target one specified type of perturbation, leaving robustness to a general family of perturbations an open and more challenging problem.

In this work, we study how to train LLM agents that remain robust to different types of policy perturbations. The difficulty is specific to the agentic setting: each action must be grounded in the environment and respect strict format constraints, and a single wrong action early in a long-horizon task can derail the whole trajectory. Therefore, perturbation injection requires careful control to preserve stability during training, and a proper perturbation scheme should avoid introducing excessive variation in policy performance. Specifically, we keep perturbations relatively mild along sensitive dimensions, since different dimensions can exhibit markedly different sensitivities to perturbations. Building on this principle, we propose \textbf{Stable Perturbation-Robust Policy Optimization (SPrPO)}, a perturbation-based method for RL training that improves the robustness of LLM agents while preserving training stability. Our main contributions are as follows:
\begin{itemize}[leftmargin=1.2em]
\item We are the first to study perturbation robustness for RL-trained LLM agents and formulate the problem under a general distribution of policy perturbations.

\item Theoretically, we extend the monotonic policy improvement analysis of TRPO~\citep{TRPO} to the perturbation setting and characterize a sufficient condition for stable policy improvement. We further establish that limit points of the resulting updates are perturbation-robust.

\item Guided by the theoretical analysis, we derive a practical condition for controlling perturbations during policy optimization and develop SPrPO.

\item We test SPrPO on ALFWorld and WebShop under different kinds of policy perturbations, such as Gaussian noise, dropout, structured and unstructured pruning, and quantization. Across these settings, SPrPO exhibits more stable optimization during training and better preserves task performance under perturbation.
\end{itemize}

%% file: sections/related_work.tex
\section{Related Work}
We briefly review prior work on robust training, with more detailed discussion in Appendix~\ref{app:more_related_works}.

\paragraph{Robust Policy in Reinforcement Learning.}
Originating from control theory, the RL community has developed the notion of robust policies as maintaining reliable performance under disturbances~\citep{robust_rl}. A broad line of work pursues this goal by deliberately introducing perturbations during training, either adversarially to improve worst-case performance~\citep{RARL,ARRL,shapo} or stochastically to improve expected performance~\citep{policy_smoothing,param_noise}. Stochastic perturbations avoid explicit inner optimization and provide a simpler training mechanism. Our method takes the same idea of injecting perturbations during training, but focuses on perturbations to the policy itself rather than noise in actions or transition dynamics, which is more appropriate for LLM agents.

\paragraph{Perturbation-based Training for LLMs.}
Introducing perturbations during language-model training has been widely explored across language understanding, instruction finetuning~\citep{neftune}, and reinforcement learning. Prior work constructs adversarial perturbations in input representations~\citep{freelb,ptp} or through a separate low-rank parameter branch~\citep{bilora}. Recent RL methods also use stochastic perturbations to promote exploration and diversify reasoning trajectories~\citep{noisyrollout,ncgrpo}. Furthermore, perturbation-based training shows improvement on robustness to quantization, through adaptation to quantization errors~\citep{tamingweights} or adaptive control of quantization-induced noise during RL~\citep{qerl}. Our work extends perturbation-based training to long-horizon LLM agents, where small policy deviations can accumulate through sequential interactions and affect future decisions. Therefore, we explicitly account for training stability, and adopt a more prudent perturbation design in practice.

%% file: sections/background.tex
\section{Background and Problem Formulation}
\label{sec:background}
In this section, we first define the notion of perturbation robustness in Subsec.~\ref{subsec:perturbation_robustness} and then review the classical TRPO lower bound and monotonic improvement theory in Subsec.~\ref{subsec:trpo_bound}.

\subsection{Problem Setup and Perturbation-Robust Policy}
\label{subsec:perturbation_robustness}
Consider a discounted Markov decision process (MDP) $\mathcal M=(\mathcal S,\mathcal A,P,r,\rho_0,\gamma)$, where $\mathcal S$ denotes the state space, $\mathcal A$ is the action space, $P(\cdot\mid s,a)$ is the transition kernel, $r:\mathcal S\times\mathcal A\to\mathbb R$ is a bounded reward function, $\rho_0$ is the initial-state distribution, and $\gamma\in(0,1)$ is the discount factor. In our setting, the agent is an LLM parameterized by $\theta$ and governed by a policy $\pi_\theta(\cdot\mid s)$. At time step $t$, the agent observes a state $s_t\in\mathcal S$, which contains the task context and interaction history available to the agent, and autoregressively generates a textual action $a_t\in\mathcal A$ according to $\pi_\theta(\cdot\mid s_t)$. The environment then returns a reward $r_t=r(s_t,a_t)$ and transitions to the next state $s_{t+1}\sim P(\cdot\mid s_t,a_t)$. The policy $\pi_\theta$ and transition kernel $P$ induce a trajectory $\tau=(s_0,a_0,s_1,a_1,\ldots)$, and the expected discounted return of $\pi_\theta$ is defined as
\begin{equation}
J(\pi_\theta) := \mathbb{E}_{\tau\sim\pi_\theta}
\left[ \sum_{t=0}^{\infty} \gamma^t r_t \right].
\label{eq:rl-objective}
\end{equation}

Let $\Xi$ denote the space of perturbations and $\mathcal P_\xi$ the perturbation operator associated with $\xi\in\Xi$. Applying $\mathcal P_\xi$ to a policy $\pi_\theta$ induces the perturbed policy $\pi_\theta^\xi := \mathcal P_\xi(\pi_\theta)$. In practice, such perturbations can be applied to policy parameters or hidden representations, including Gaussian noise, random dropping, pruning, and quantization. Let $\rho(\xi)$ measure the magnitude of a perturbation $\xi$, and let $\Xi_\delta := \{\xi\in\Xi : \rho(\xi)\le\delta\}$ denote the set of perturbations within radius $\delta$. We define perturbation robust policy as follows.

\begin{definition}[Perturbation robustness]
\label{def:perturbation_robustness}
Let $\nu=\nu_\delta$ be a perturbation distribution supported on $\Xi_\delta$, and denote the expected perturbed return by
$\bar J_{\nu_\delta}(\theta)
:=
\mathbb E_{\xi\sim\nu_\delta}[J(\pi_\theta^\xi)]$.
For $\epsilon\ge0$ and $\beta\in[0,1]$, a policy $\pi_\theta$ is
$(\delta,\epsilon,\beta)$-probabilistic perturbation robust under
$\nu_\delta$ if
\begin{equation}
\Pr_{\xi\sim\nu_\delta}
\left( J(\pi_\theta^\xi) \ge \bar J_{\nu_\delta}(\theta) - \epsilon \right)
\ge 1-\beta.
\label{eq:robustness}
\end{equation}
Moreover, it is $(\delta,\epsilon)$-uniform perturbation robust if for any $\xi\in\Xi_\delta$, we have
$
J(\pi_\theta^\xi)
\ge
\bar J_{\nu_\delta}(\theta)-\epsilon.
$
\end{definition}
Uniform perturbation robustness is stronger, but can be overly restrictive in practice. Instead, probabilistic robustness provides a more practical criterion: under perturbations of magnitude at most $\delta$, we allow a performance degradation larger than $\epsilon$ only with a small probability $\beta$, which is the objective we seek to achieve in this work.

\subsection{TRPO Lower Bound and Monotonic Improvement Property}
\label{subsec:trpo_bound}

Consider one step policy update from $\theta$ to $\theta'$, whose improvement is measured by $\Delta J:=J(\pi_{\theta'})-J(\pi_\theta)$. Classical trust region policy improvement analysis \citep{TrustRegionLearning, TRPO} studies conditions that guarantee positive improvement, i.e., $\Delta J>0$. The following lemma, as a central result in TRPO analysis, provides a lower bound on the policy improvement $\Delta J$.

\begin{lemma}[TRPO lower bound {\cite[Theorem~1]{TRPO}}]
\label{lem:trpo_improvement}
For a policy $\pi$, let
$d_\pi(s)=(1-\gamma)\sum_{t\ge0}\gamma^t\Pr_\pi(s_t=s)$ denote its normalized discounted state visitation distribution, and let $D_{\mathrm{TV}}(p,q)=\frac12\sum_a|p(a)-q(a)|$. For any two policies $\pi,\pi'$ satisfying $\sup_{s,a}|A_\pi(s,a)|\le A_{\max}$, define $\alpha=\sup_s D_{\mathrm{TV}}(\pi(\cdot\mid s),\pi'(\cdot\mid s))$. Then
\begin{equation}
\Delta J(\pi',\pi)
\ge
\frac{1}{1-\gamma}
\mathbb E_{s\sim d_\pi,\,a\sim\pi'(\cdot\mid s)}
[A_\pi(s,a)]
-
\frac{4\gamma A_{\max}}{(1-\gamma)^2}\alpha^2.
\label{eq:trpo_improvement_bound}
\end{equation}
\end{lemma}

Lemma~\ref{lem:trpo_improvement} gives the standard trust region interpretation of policy improvement. Given the current policy $\pi_\theta$, denote the first term by $\mathcal L(\theta')$, where $\theta'$ parameterizes a candidate updated policy. This term measures the surrogate improvement relative to $\pi_\theta$, while the second term penalizes excessive policy deviation. Therefore, as long as this deviation is controlled, increasing $\mathcal L(\theta')$ guarantees monotonic improvement in the true return. 

%% file: sections/theoretical_analysis.tex
\section{Theoretical Analysis}

In this section, we extend the classical monotonic improvement analysis to the perturbation setting. Subsec.~\ref{subsec:perturbed_policy_improvement} establishes stable policy improvement under perturbations and derives a practical condition for perturbation allocation. Subsec.~\ref{subsec:stable_convergence} further shows that, under suitable regularity conditions, the resulting policy updates admit perturbation-robust limit points.

\subsection{Policy Improvement under Perturbations}
\label{subsec:perturbed_policy_improvement}
For perturbation $\xi$, let $\pi_\theta^\xi$ and $\pi_{\theta'}^\xi$ denote the perturbed policies before and after the update, respectively. We define the perturbed policy improvement as
\begin{equation}
\Delta J_\xi
=
\Delta J(\pi_{\theta'}^\xi;\pi_\theta^\xi)
:=
J(\pi_{\theta'}^\xi)-J(\pi_\theta^\xi).
\label{eq:pertubed_policy_improve}
\end{equation}
The quantity $\Delta J_\xi$ therefore describes how the benefit of the update varies with perturbation, with $\Delta J_0=\Delta J$ recovering the standard policy improvement. Now we define the expected perturbed surrogate improvement as follows.

\begin{definition}[Surrogate for perturbed policy improvement]
\label{def:perturbed_surrogate}
For the current policy parameter $\theta$ and a candidate updated parameter $\theta'$, we define the policy improvement surrogate under a perturbation $\xi$ as
\begin{equation}
\mathcal L_\xi(\theta')
:=
\frac{1}{1-\gamma}
\mathbb E_{\substack{
s\sim d_{\pi_\theta^\xi}\\
a\sim\pi_\theta^\xi(\cdot\mid s)}}
\left[
\frac{\pi_{\theta'}^\xi(a\mid s)}
{\pi_\theta^\xi(a\mid s)}
A_{\pi_\theta^\xi}(s,a)
\right].
\label{eq:perturbed_surrogate_single}
\end{equation}
Then $\mathcal L_\nu(\theta')
:=
\mathbb E_{\xi\sim\nu}[\mathcal L_\xi(\theta')]$ denotes the expected policy improvement surrogate, and the corresponding policy deviation is $\alpha_\xi(\theta')
:=
\sup_s
D_{\mathrm{TV}}\!\left(
\pi_\theta^\xi(\cdot\mid s),
\pi_{\theta'}^\xi(\cdot\mid s)
\right)$.
\end{definition}

We emphasize that $\mathcal L_\nu(\theta')$ is estimated under the state visitation distribution of the perturbed current policy $\pi_\theta^\xi$, so optimizing it promotes policy improvement in expectation over perturbations. However, a positive expected improvement does not preclude degradation under a non-negligible subset of perturbations, as the realized improvement can vary substantially across perturbation realizations. Therefore, we seek a stronger stability guarantee: under the perturbation distribution, the updated policy should improve with high probability.

\begin{theorem}[Stable policy improvement]
\label{thm:general_perturbation_improvement}
Suppose that $\sup_{s,a}|A_{\pi_\theta^\xi}(s,a)|\le A_{\max}$ for $\nu$-almost every $\xi$, and that $\operatorname{Var}_\nu[\Delta J_\xi]\le v_\nu<\infty$. Define the expected improvement lower bound
\begin{equation}
m_\nu
:=
\mathcal L_\nu(\theta')
-
\frac{4\gamma A_{\max}}{(1-\gamma)^2}
\mathbb E_{\xi\sim\nu}[\alpha_\xi^2].
\label{eq:mean_improvement_margin}
\end{equation}
Then, for any $\beta\in(0,1)$,
\begin{equation}
\Pr_{\xi\sim\nu}\!\left(
\Delta J_\xi
\ge
m_\nu
-
\sqrt{\frac{1-\beta}{\beta}\,v_\nu}
\right)
\ge
1-\beta.
\label{eq:general_probabilistic_improvement}
\end{equation}
Consequently, if $m_\nu\ge\sqrt{(1-\beta)v_\nu/\beta}$, then $\Pr_{\xi\sim\nu}(\Delta J_\xi\ge0)\ge1-\beta$.
\end{theorem}

\begin{figure*}[t]
\centering
\includegraphics[width=\textwidth]{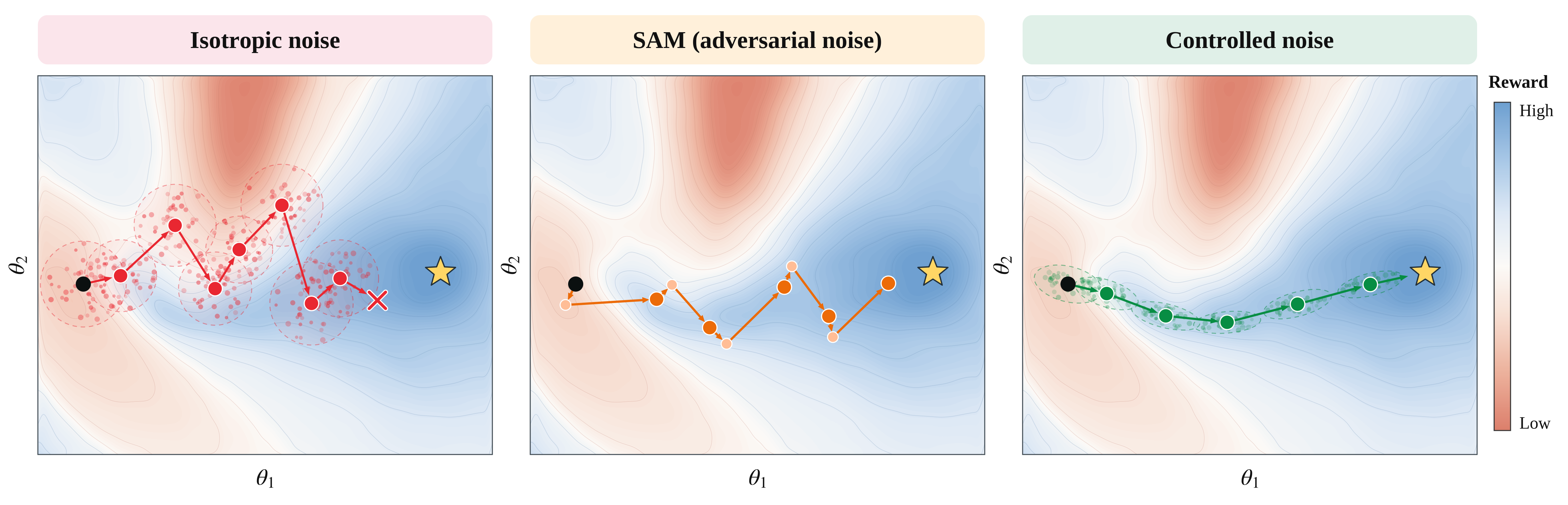}
\caption{Illustrative example: different perturbation strategies can lead to distinct training dynamics during policy optimization. Isotropic noise (left) applies the same perturbation scale across directions, whereas SAM-style perturbations (middle) adversarially push the policy along the sensitive direction. Both can disrupt the optimization trajectory and undermine training stability. Controlled noise (right) instead limits perturbation magnitude along sensitive directions to reduce the induced variance in policy improvement and promote a more stable training process.}
\label{fig:optimization}
\end{figure*}

Theorem~\ref{thm:general_perturbation_improvement} shows that stable policy improvement depends not only on the expected improvement, but also on the variance induced by perturbations. As is illustrated in Figure~\ref{fig:optimization}, perturbations along sensitive directions can divert the optimization trajectory toward low-reward regions, motivating a practical way to better control variance during training. However, the perturbation distribution encountered at deployment is generally unknown, making the improvement variance difficult to characterize directly. Motivated by a common modeling principle in distributionally robust optimization, where tractable nominal distributions are used as reference for uncertain and potentially non-Gaussian disturbances~\citep{dro1,dro2}, we use Gaussian perturbations as a tractable proxy during training. Under this choice, the variance admits an explicit sensitivity-based control through dimension-dependent perturbation scales, leading to the following corollary.

\begin{corollary}[Stable policy improvement under Gaussian perturbations]
\label{cor:gaussian_channel_improvement}
Under the setting of Theorem~\ref{thm:general_perturbation_improvement}, let $\nu=\mathcal N(0,\Sigma)$ with $\Sigma=\operatorname{diag}(\sigma_1^2,\ldots,\sigma_d^2)$. Suppose that $\mathcal L_\xi(\theta')$ is continuously differentiable with respect to $\xi$. For each perturbation dimension $i$, define the sensitivity as 
$S_i(\theta')
:=
\mathbb E_{\xi\sim\nu}
[
|\partial \mathcal L_\xi(\theta')/\partial \xi_i|^2
]
<\infty$, and the supremum policy drift as
$\bar\alpha:=\sup_\xi \alpha_\xi$. If
$\mathcal L_\nu(\theta')
-
\frac{4\gamma A_{\max}}{(1-\gamma)^2}\bar\alpha^2>0$
and the following stability condition holds:
\begin{equation}
\sum_{i=1}^d
\sigma_i^2 S_i(\theta')
\le
\frac{\beta}{1-\beta}
\left(
\mathcal L_\nu(\theta')
-
\frac{4\gamma A_{\max}}{(1-\gamma)^2}\bar\alpha^2
\right)^2,
\label{eq:gaussian_stable_improvement_condition}
\end{equation}
then
$\Pr_{\xi\sim\nu}(\Delta J_\xi\ge0)\ge1-\beta$.
\end{corollary}

We note that the condition in Corollary~\ref{cor:gaussian_channel_improvement} gives a tractable characterization of Gaussian perturbations that admit stable policy improvement. In particular, the guarantee is jointly determined by the surrogate improvement, the policy drift, and the perturbation-induced variance. This suggests that desirable perturbations should allocate smaller variance to more sensitive dimensions, and can tolerate larger variance when the surrogate improvement is larger.

\subsection{Convergence under Stable Policy Updates}
\label{subsec:stable_convergence}

We next establish convergence under stable policy updates. Given $\epsilon>0$ and $\bar\beta\in(0,1)$, we aim to obtain a policy whose perturbed return falls more than $\epsilon$ below its mean with probability at most $\bar\beta$. Let $C_J:=J_{\max}-J_{\min}$ denote the range of the bounded return. Assume that the initial policy has sufficiently small return variance when excluding a negligible subset.

\begin{assumption}[Proper initial variance]
\label{ass:initial_variance}
For the initial policy $\pi_0\in\Pi$, there exists a measurable subset $\Xi_0\subseteq\Xi_\delta$ with $\nu(\Xi_0)\ge1-\beta_0$ such that $\operatorname{Var}_{\xi\sim\nu(\cdot\mid\Xi_0)}[J(\pi_0^\xi)]\le\sigma_0^2$, where $\beta_0\in[0,1)$ and $\sigma_0\ge0$ satisfy $\beta_0\le\frac12\min\{\bar\beta,\epsilon/C_J\}$ and $\sigma_0<\frac14\epsilon\sqrt{\bar\beta}$.
\end{assumption}

At iteration $k$, let $m_{\nu,k}(\pi')$ denote the expected improvement lower bound for a candidate policy $\pi'$ relative to the current policy $\pi_k$. Following the stability condition in Theorem~\ref{thm:general_perturbation_improvement}, define
\begin{equation}
\mathcal{U}_{\beta}(\pi_k)
:= \left\{ \pi' \in \Pi :\;
m_{\nu,k}(\pi') \ge
\sqrt{ \frac{1-\beta}{\beta}\,
\operatorname{Var}_{\xi\sim\nu}\!\left[ \Delta J\bigl((\pi')^\xi;\pi_k^\xi\bigr) \right] }
\right\}.
\label{eq:safe-set}
\end{equation}
Starting from $\pi_0$, update the policy by selecting the best $\pi_{k+1}\in\mathcal U_\beta(\pi_k)$ at each iteration. The following theorem shows that, as long as the probability of degradation at each step is sufficiently small, stable policy
updates can lead to perturbation robust policies.

\begin{theorem}[Convergence of stable policy optimization]
\label{thm:robust_policy_limit_points}
Suppose that the policy space $\Pi$ is compact and that
$\pi\mapsto J(\pi^\xi)$ is continuous for $\nu$-almost every $\xi$. Under Assumption~\ref{ass:initial_variance}, if $\beta\le\epsilon^2\bar\beta/(100C_J^2+\epsilon^2\bar\beta)$, the policy sequence $\{\pi_k\}_{k\ge0}$ has a nonempty set of limit points. Moreover, every limit point $\bar\pi$ is $(\delta,\epsilon,\bar\beta)$-probabilistic perturbation robust.
\end{theorem}

%% file: sections/method.tex
\section{Practical Algorithm for Stable Perturbation-Robust Policy Optimization}
\label{sec:practical_algorithm}
\begin{algorithm}[b]
\caption{Stable Perturbation-Robust Policy Optimization (SPrPO)}
\label{alg:prpo}
\small
\begin{algorithmic}[1]
\Require Initial policy $\theta_1$, iterations $K$, base perturbation
$\rho$, sensitivity exponent $\alpha$
\State Initialize $\widehat S_{1,\ell c}\gets 1$ and $\widehat m_1\gets 1$
\For{$k=1,\ldots,K$}
    \State $\displaystyle
    \sigma_{k,\ell c}
    \gets
    \rho\,\widehat m_k\,
    \widehat S_{k,\ell c}^{-\alpha/2}$
    \State Sample
    $\xi_{k,\ell c}\sim\mathcal N(0,\sigma_{k,\ell c}^2)$
    \State Collect
    $\mathcal D_k\sim\pi_{\theta_k}^{\xi_k}$
    and estimate advantages $A_i$ for $i\in\mathcal D_k$
    \State Update the policy:
    $\displaystyle
    \theta_{k+1}
    \gets
    \theta_k
    +
    \eta_k
    \nabla_\theta
    \mathcal L_k(\theta_k;\xi_k)$
    \State $\displaystyle
    \widehat S_{k+1,\ell c}
    \gets
    \left|
    \frac{\partial
    \mathcal L_k(\theta_{k+1};\xi_k)}
    {\partial \xi_{k,\ell c}}
    \right|^2, \quad \displaystyle
    \widehat m_{k+1}
    \gets
    \left\|
    \nabla_\theta
    \mathcal L_k(\theta_{k+1};\xi_k)
    \right\|_2^2$
\EndFor
\State \Return $\theta_{K+1}$
\end{algorithmic}
\end{algorithm}
Building on the preceding analysis, we develop a practical realization of the stability condition for LLM policy optimization. Specifically, we apply controlled Gaussian perturbations in the hidden space and estimate the sensitivity and improvement statistics from tractable rollout samples, leading to \textbf{Stable Perturbation-Robust Policy Optimization (SPrPO)}, as summarized in Algorithm~\ref{alg:prpo}.

\paragraph{Multiplicative Gaussian perturbations to FFN channels.} In modern Transformer architectures, feed-forward networks (FFNs) account for a large fraction of the model parameters, operate independently on each token representation, and have been shown to store rich textual and semantic patterns~\citep{Transformer, FFN_importance, qwen25}. In the Qwen2.5 architecture, each FFN consists of an up-projection, a gated SiLU activation, and a down projection. To be specific, given an input $x_\ell$ in layer $\ell$, it is computed as
\[
h_\ell
=
\operatorname{SiLU}(W_{\mathrm{gate},\ell}x_\ell)
\odot
W_{\mathrm{up},\ell}x_\ell,
\qquad
y_\ell
=
W_{\mathrm{down},\ell}h_\ell,
\]
Instead of modifying the full parameter space, we perturb the intermediate FFN channel $h_{\ell}$ with multiplicative Gaussian noise $\xi_{\ell}$, which has been shown to be an effective mechanism for improving robustness~\citep{multiplicative_perturbation}. Let $\xi_{\ell c}\sim\mathcal N(0,\sigma_{\ell c}^{2})$ denote the Gaussian perturbation for channel $c$, i.e., coordinate $c$ of $h_\ell$, the perturbed FFN forward pass is $\widetilde h_{\ell c}=(1+\xi_{\ell c})h_{\ell c}$. Since $\xi_{\ell}$ is applied after the up projection and activation, it is equivalent to perturbing the corresponding columns of down projection matrix $W_{\mathrm{down},\ell}$, while offering a simpler implementation and more stable training.

\paragraph{Adaptive and sensitivity-aware Gaussian perturbations.} Motivated by the stability condition in Corollary~\ref{cor:gaussian_channel_improvement}, we adopt
non-isotropic Gaussian perturbations, i.e., $\sigma_{\ell c}$ are non-identical across different channels. However, unlike supervised fine-tuning, where loss-related statistics can be evaluated on a fixed training set~\citep{gradient_sensitivity, tamingweights}, the exact channel sensitivities and expected improvement in reinforcement learning are difficult to evaluate because the data distribution changes with the policy. Therefore in practice, we estimate the quantities by reusing the rollout batch from the current policy update. At iteration $k$, we collect trajectories $\mathcal D_k\sim\pi_{\theta_k}^{\xi_k}$ and estimate advantages $A_i$ for
samples $i\in\mathcal D_k$, and optimize the policy using perturbed loss surrogate with PPO clipping
\begin{equation*}
\mathcal{L}_k(\theta;\mathcal{\xi}_{k})
=
\frac{1}{|\mathcal D_k|}
\sum_{i\in\mathcal D_k}
\min\left(
r_i(\theta) A_i,\,
\operatorname{clip}
\bigl(r_i(\theta),1-\epsilon,1+\epsilon\bigr)A_i
\right),
\label{eq:practical_ppo}
\end{equation*}
where $r_i(\theta)
=
\pi_{\theta}^{\xi_k}(a_i\mid s_i)/
\pi_{\theta_k}^{\xi_k}(a_i\mid s_i)$. After obtaining $\theta_{k+1}$, we replay $\mathcal D_k$ for one additional
forward and backward pass without extra rollouts, and estimate the channel sensitivity and expected improvement as

\begin{equation*}
\widehat S_{k+1,\ell c}
=
\left|
\frac{\partial
\mathcal L_k(\theta_{k+1};\xi_k)}
{\partial \xi_{k,\ell c}}
\right|^2,
\qquad
\widehat m_{k+1}
=
\left\|
\nabla_\theta
\mathcal L_k(\theta_{k+1};\xi_k)
\right\|_2^2.
\end{equation*}
Here, $\widehat S_{k+1,\ell c}$ measures the local sensitivity of channel $(\ell,c)$. For the expected improvement $\widehat m_{k+1}$, we use the squared PPO gradient norm as a practical proxy~\citep{gradient_norm_ada, gradient_norm_flat}, since under a small gradient ascent step, the first-order increase is proportional to the squared gradient norm. Finally, we choose standard deviation of channel perturbation to increase with the estimated improvement and decrease with channel sensitivity, i.e.,
$\sigma_{k+1,\ell c}
=
\rho\,\widehat m_{k+1}\,
\widehat S_{k+1,\ell c}^{-\alpha/2}$,
which follows the allocation principle implied by the stability condition~\eqref{eq:gaussian_stable_improvement_condition}. Notably, the hyperparameter $\rho>0$ controls the base perturbation scale, while $\alpha\ge0$ controls the strength of sensitivity adaptation: $\alpha=0$ reduces to sensitivity-independent noise, whereas $\alpha=1$ corresponds to inverse-square-root sensitivity scaling.

%% file: sections/experiments.tex
\section{Experiments}
\label{sec:experiments}
In this section, we empirically validate the effectiveness of SPrPO. Subsec.~\ref{sec:main_results} evaluates robustness under diverse policy perturbations, Subsec.~\ref{subsec:reward_landscape} visualizes the reward landscape, and Subsec.~\ref{sec:ablation} studies the effects of the main algorithmic components and hyperparameters. Full details of experiment settings are provided in Appendix~\ref{app:exp_details}.

\subsection{Experiment Setup}
\label{sec:experimental_setup}

\paragraph{Environments.} We train and evaluate the LLM agent on two commonly used long-horizon benchmarks, ALFWorld~\citep{alfworld} and WebShop~\citep{webshop}. ALFWorld is a text-based embodied environment, where an agent takes several steps to complete household tasks. WebShop is an interactive web shopping environment in which an agent searches for and purchases products according to the given instructions. In both environments, the agent receives textual observations and executes sequences of structured actions under strict environment parsers, with ALFWorld allowing up to 50 steps per episode and WebShop up to 15 steps.

\paragraph{Baselines.} For all experiments, we use Qwen2.5-1.5B-Instruct~\citep{qwen25} as the base model and GiGPO~\citep{gigpo} as the base training algorithm, which is a critic-free RL method designed to better handle sparse trajectory-level rewards. We refer to the version without training-time perturbations as \textbf{Vanilla}. To our knowledge, we are the first to explicitly study policy perturbations during RL training for long-horizon LLM agents, while prior robust RL methods for continuous control do not directly transfer to this setting. We therefore consider two representative perturbation baselines: \textbf{Multiplicative Gaussian}, which applies same Gaussian variance across all channels, and \textbf{SAM}~\citep{SAM}, which uses adversarial perturbations along the PPO loss increasing direction. Both are applied to FFN channel space in a multiplicative manner as our method, providing controlled comparisons between different perturbation strategies.

\paragraph{Evaluation protocol under perturbations.}
We evaluate agent performance by task success rate under different policy perturbations. For WebShop, we evaluate on a set of 500 shopping tasks. For ALFWorld, we evaluate on all 274 validation tasks, including 140 in-distribution (ID) tasks and 134 out-of-distribution (OOD) tasks. We consider several representative perturbations that differ in both structure and granularity:
\begin{itemize}[leftmargin=*,itemsep=2pt,topsep=2pt]

\item \textbf{Channel-wise perturbations.} These perturbations operate at the level of FFN hidden channels, which is equivalent to modifying entire rows or columns of the corresponding weight matrices. \emph{Gaussian perturbation} applies multiplicative noise $\mathcal{N}(0,\sigma^2)$ in the same manner as during training. \emph{Dropout}~\citep{dropout} randomly removes FFN channels with probability $p$. Similarly, \emph{structured pruning} removes entire FFN channels based on their Fisher importance~\citep{fisher_pruning}. Specifically, we use the same calibration set for all evaluated methods for fair comparison and estimate $\mathbb{E}[(h_{\ell,c}\,\partial \mathcal{L}/\partial h_{\ell,c})^2]$ as the importance score.

\item \textbf{Weight space perturbations.}
Unlike channel-wise perturbations, these perturbations operate at the level of individual weights, with a different pattern from the channel-wise perturbations used during training. We use SparseGPT~\citep{sparseGPT} for \emph{unstructured pruning}, which uses second-order information to identify and compensate pruned weights, and apply it only to the FFN linear weights. For \emph{quantization}, we use round-to-nearest (RTN) weight quantization~\citep{rtn_quantization} over all transformer linear weights at different bit widths.

\end{itemize}

\subsection{Main Results: Robustness to Diverse Policy Perturbations}
\label{sec:main_results}
\input{images/main_results/figure.tex}
\input{images/main_results/tables/main_table.tex}

We evaluate the robustness of the policies under diverse perturbations. Figure~\ref{fig:main-robustness} reports the full robustness curves as the perturbation strength increases, while Table~\ref{tab:main-robustness} summarizes the results under relatively strong perturbation settings for direct comparison. Overall, policies trained by our method are generally more robust than all baselines. At sufficiently extreme perturbation levels, however, all methods can eventually collapse. For example, under INT4 quantization on WebShop, success rates fall to the single digits, where differences between policies become less informative and performance is close to the failure regime.

\subsection{Reward Landscape Visualization}
\label{subsec:reward_landscape}
\input{images/analysis/reward_surface/reward_surface.tex}

To further examine local sensitivity to perturbations, we visualize reward surfaces, following prior work on reward landscape analysis~\citep{cliff_diving,flat_reward}. For each trained policy, we sample two fixed multiplicative Gaussian perturbation directions, $\xi_1$ and $\xi_2$, and vary their magnitudes $w_1$ and $w_2$ over a two-dimensional grid. Each point corresponds to the policy under perturbation $w_1\xi_1 + w_2\xi_2$, with the $z$-axis representing the success rate, as shown in Figure~\ref{fig:reward-surface}.

\subsection{Ablation and Hyperparameter Study}
\label{sec:ablation}

\begin{figure*}[htbp]
\centering
\includegraphics[width=0.5\textwidth,trim=0 0 9.8in 0,clip]{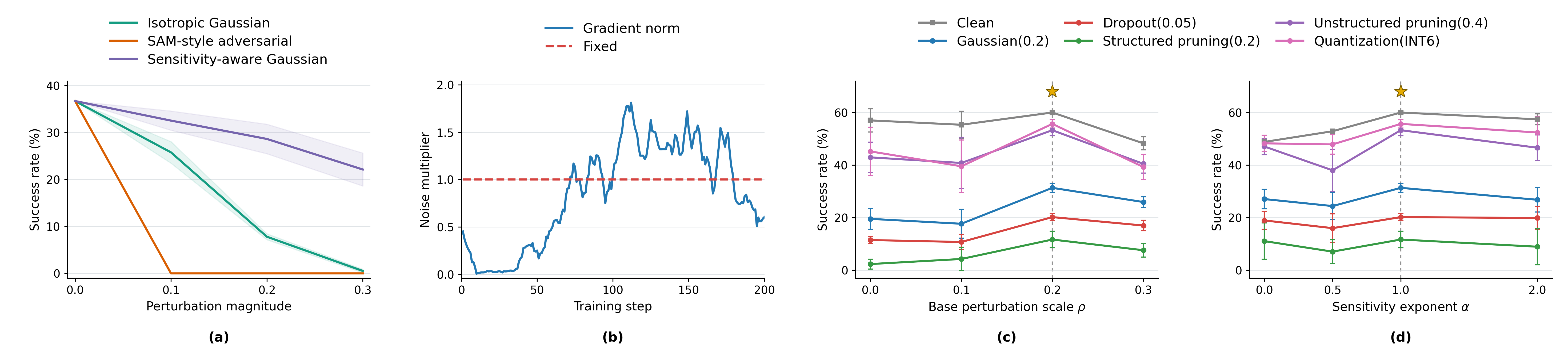}%
\includegraphics[width=0.5\textwidth]{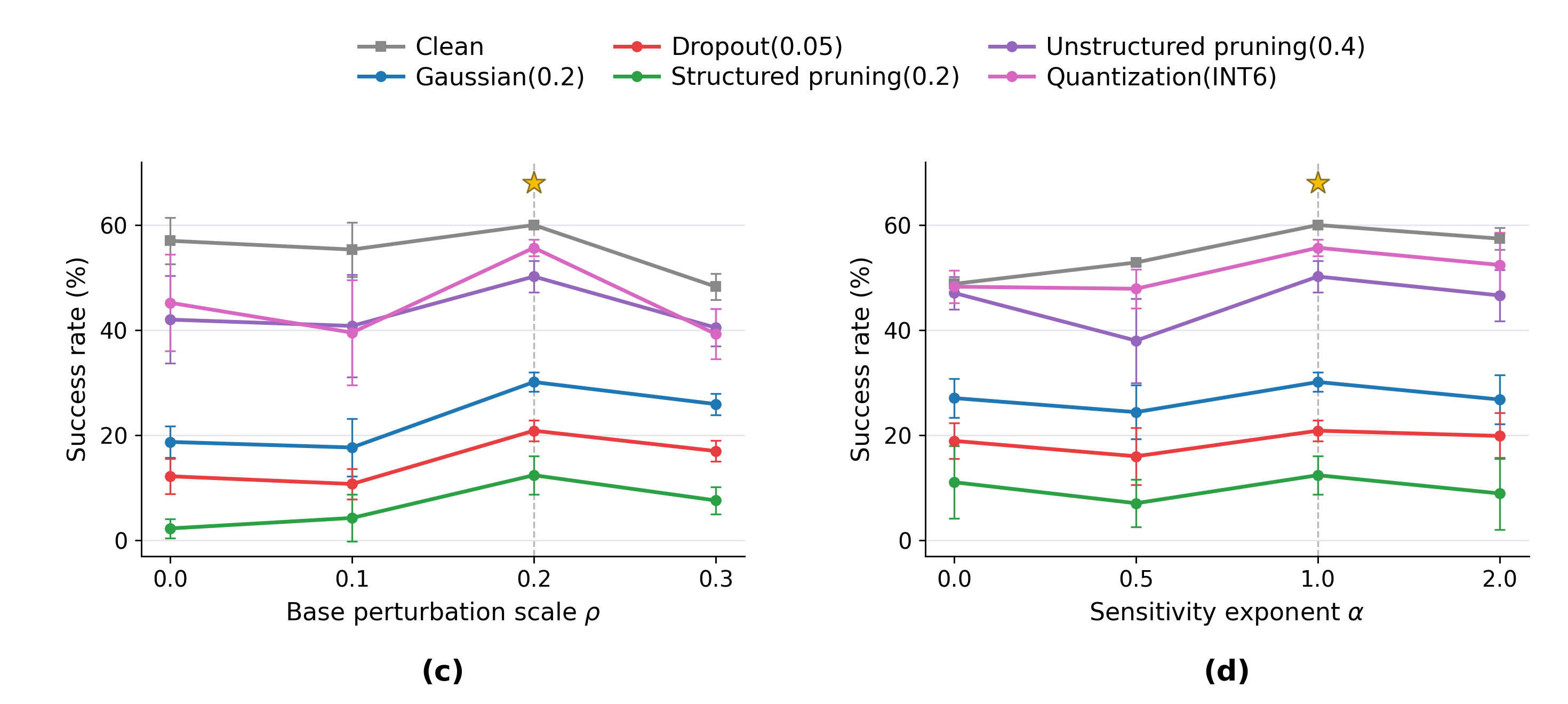}
\caption{Ablation and hyperparameter study.
(a) Comparison of perturbation strategies during training.
(b) Effect of gradient-norm scaling.
(c,d) Effects of perturbation scale $\rho$ and sensitivity exponent $\alpha$.
Starred and $\rho=0$ results reuse the main-experiment estimates in Table~\ref{tab:main-robustness}.}
\label{fig:ablation}
\end{figure*}

Figure~\ref{fig:ablation}(a,b) examines the two main design choices of SPrPO. Under the same perturbation magnitude, SAM-style perturbations collapse task performance, and isotropic Gaussian noise also causes substantial degradation. In contrast, reducing perturbations on sensitive channels preserves more original performance. Since SAM is substantially more aggressive at the same magnitude, we further reduce its perturbation scale to obtain a stable training baseline. Gradient-norm scaling increases noise strength as training progresses; removing it reduces the clean success rate by $7.87$ percentage points (Table~\ref{tab:more-components}).

Figure~\ref{fig:ablation}(c,d) shows the effect of hyperparameters $\rho$ and $\alpha$. A small $\rho$ provides insufficient perturbation exposure to improve robustness, whereas an overly large $\rho$ corrupts the training signal and causes the model to underfit. For $\alpha$, a small value does not provide enough protection, leaving sensitive channels over-perturbed, and further increasing $\alpha$ also reduces the robustness gains, since sensitive channels become barely perturbed during training.

%% file: images/main_results/figure.tex
\begingroup
\newlength{\MainPanelWidth}
\newlength{\MainPanelGap}
\begin{figure*}[t]
  \centering
  \setlength{\MainPanelWidth}{0.2455\textwidth}
  \setlength{\MainPanelGap}{0.006\textwidth}
  \captionsetup[subfigure]{font=footnotesize,skip=1pt,justification=centering}
  \begin{tabular}{@{}c@{\hspace{\MainPanelGap}}c@{\hspace{\MainPanelGap}}c@{\hspace{\MainPanelGap}}c@{}}
      \begin{subfigure}[t]{\MainPanelWidth}
    \includegraphics[width=\linewidth]{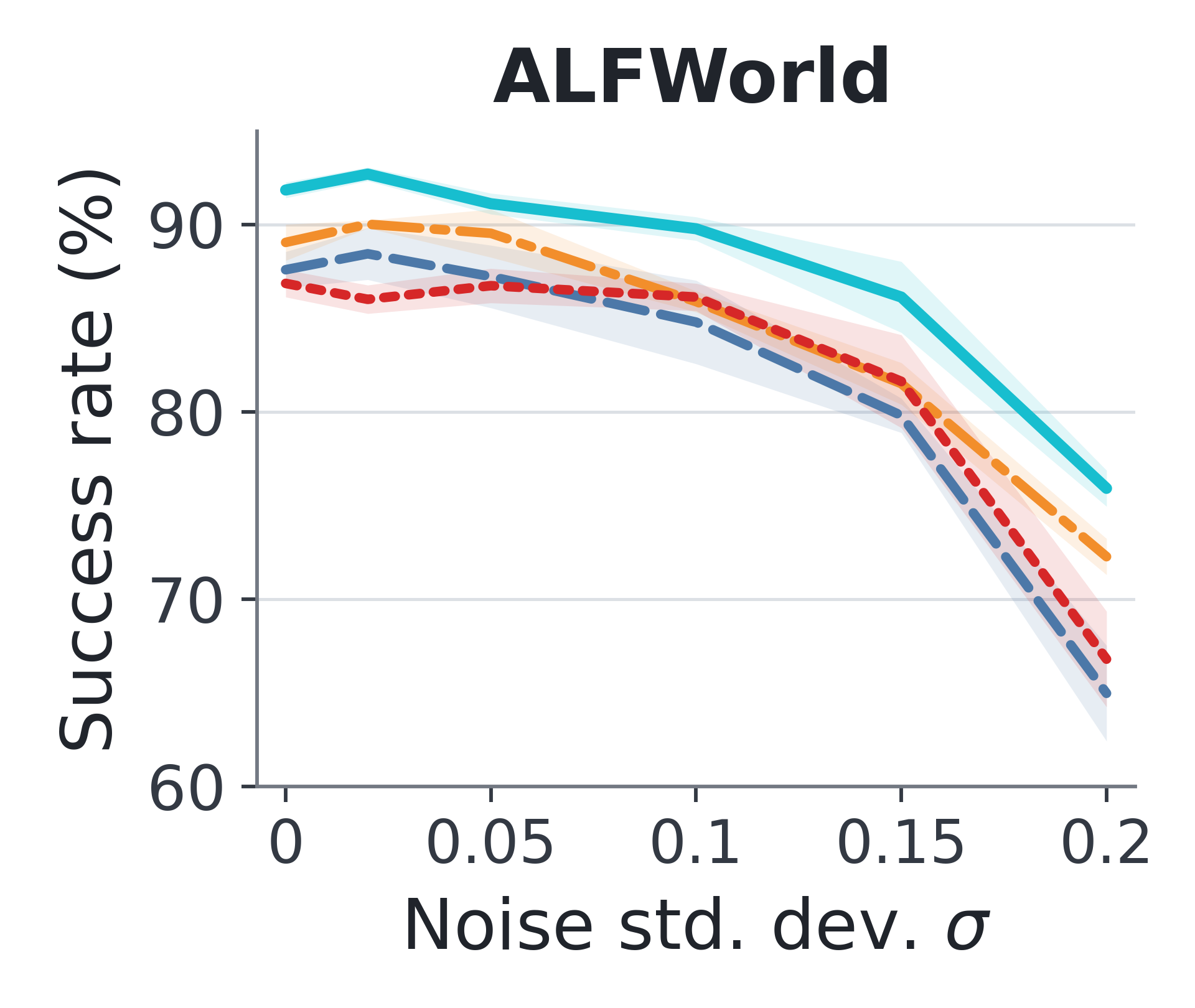}
    \caption{Gaussian noise}\label{fig:main-alfworld-gaussian}
  \end{subfigure}%
 &
      \begin{subfigure}[t]{\MainPanelWidth}
    \includegraphics[width=\linewidth]{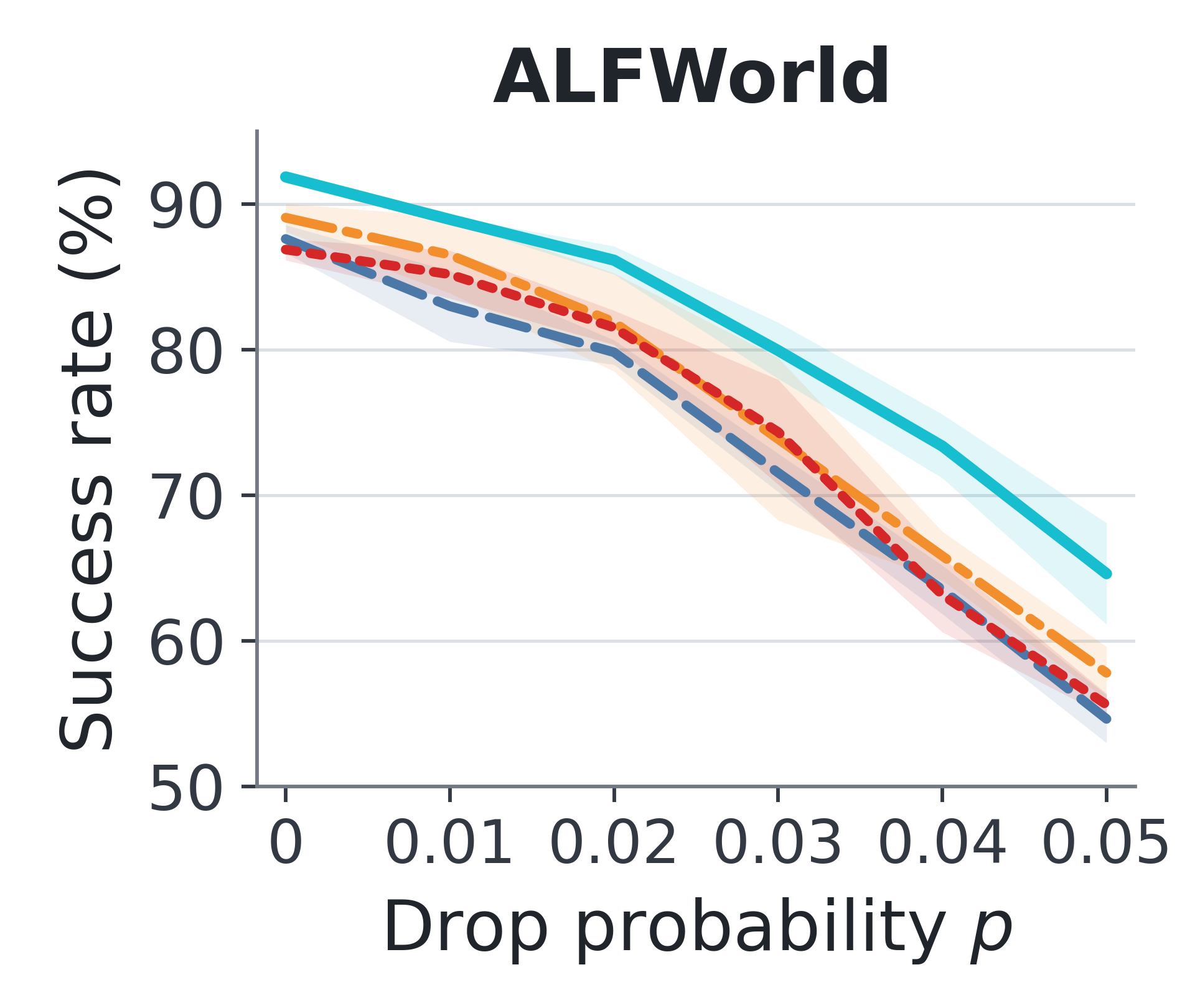}
    \caption{Channel dropout}\label{fig:main-alfworld-dropout}
  \end{subfigure}%
 &
      \begin{subfigure}[t]{\MainPanelWidth}
    \includegraphics[width=\linewidth]{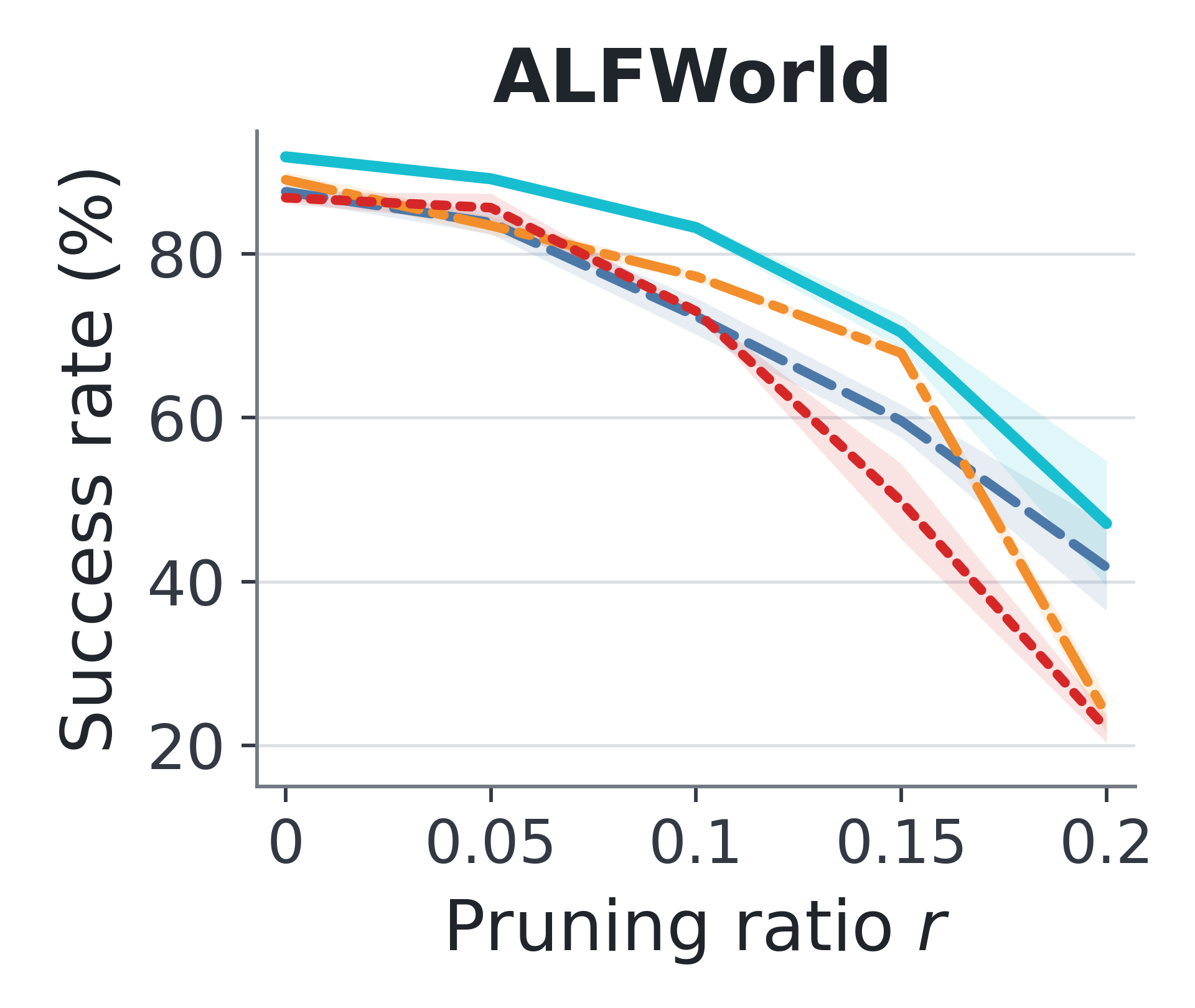}
    \caption{Structured FFN pruning}\label{fig:main-alfworld-pruning}
  \end{subfigure}%
 &
      \begin{subfigure}[t]{\MainPanelWidth}
    \includegraphics[width=\linewidth]{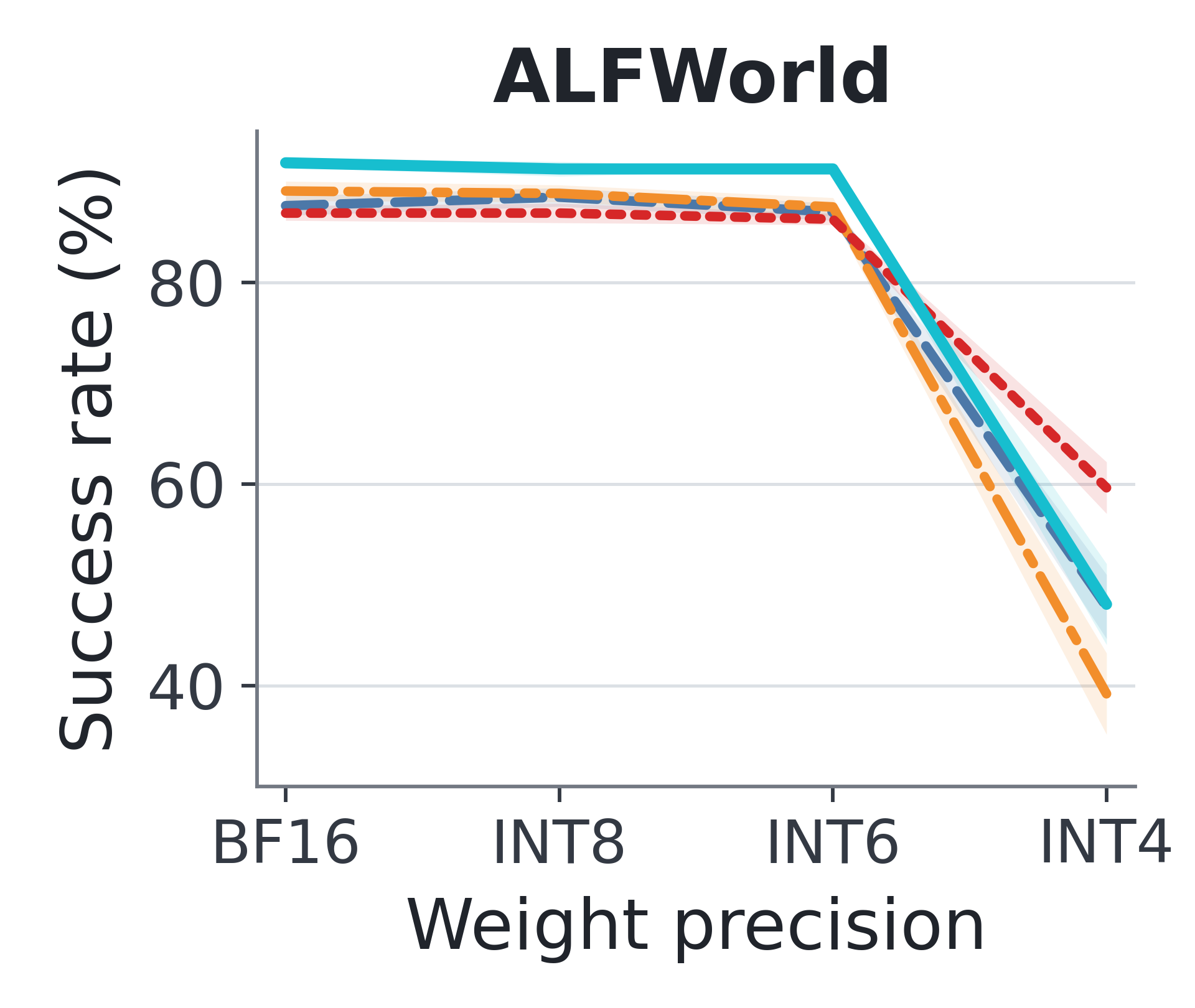}
    \caption{Weight quantization}\label{fig:main-alfworld-quantization}
  \end{subfigure}%
 \\[1pt]
      \begin{subfigure}[t]{\MainPanelWidth}
    \includegraphics[width=\linewidth]{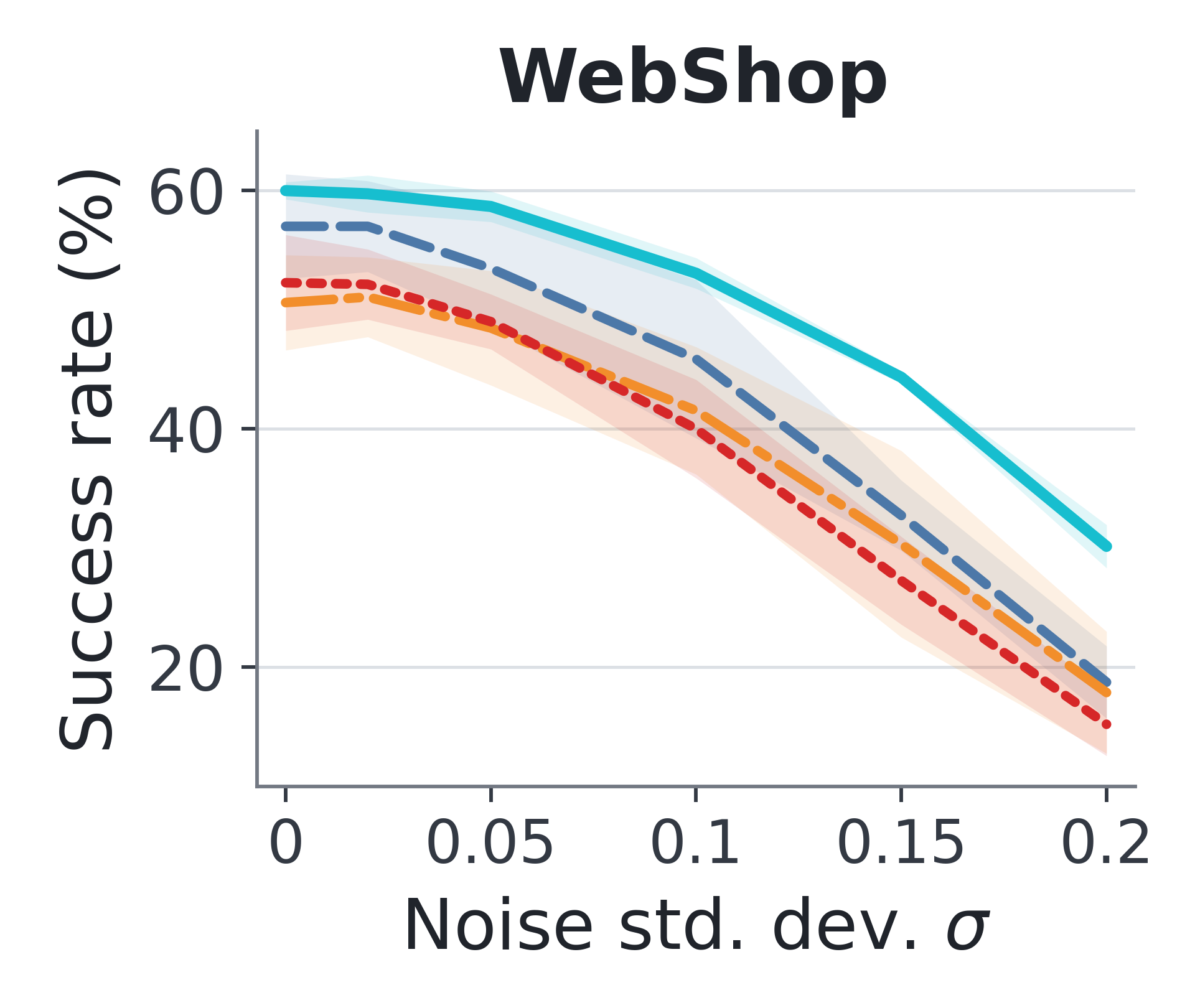}
    \caption{Gaussian noise}\label{fig:main-webshop-gaussian}
  \end{subfigure}%
 &
      \begin{subfigure}[t]{\MainPanelWidth}
    \includegraphics[width=\linewidth]{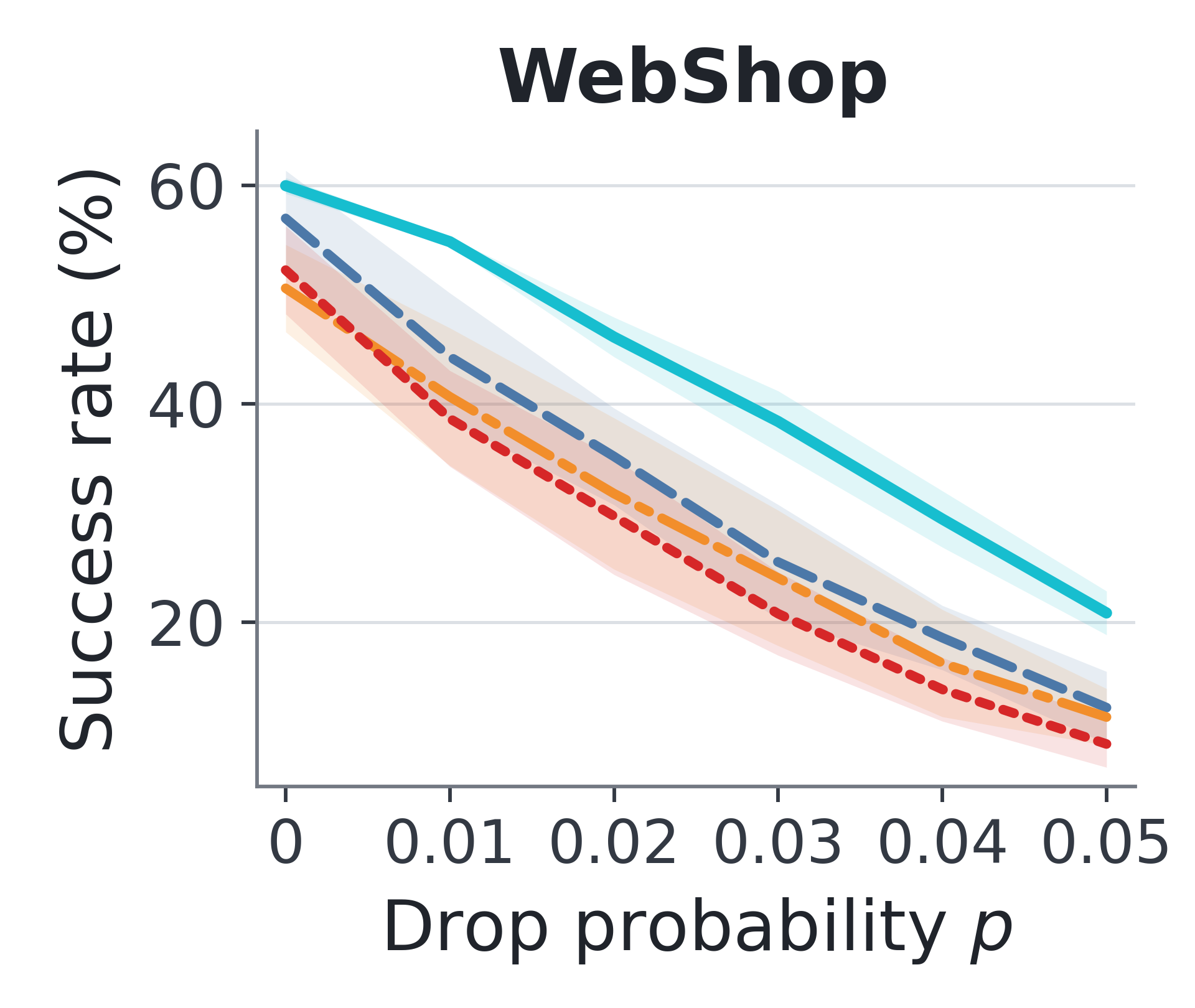}
    \caption{Channel dropout}\label{fig:main-webshop-dropout}
  \end{subfigure}%
 &
      \begin{subfigure}[t]{\MainPanelWidth}
    \includegraphics[width=\linewidth]{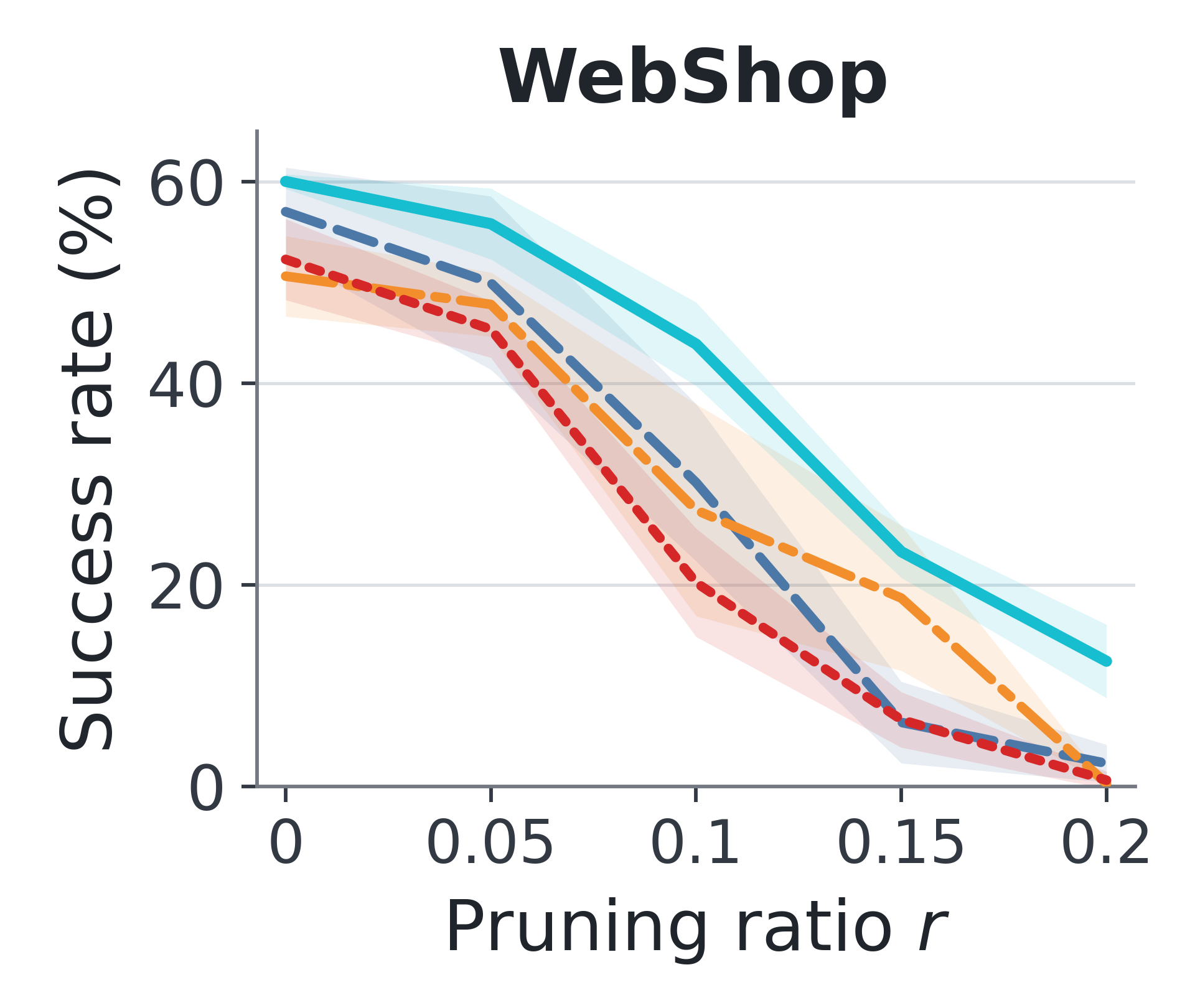}
    \caption{Structured FFN pruning}\label{fig:main-webshop-pruning}
  \end{subfigure}%
 &
      \begin{subfigure}[t]{\MainPanelWidth}
    \includegraphics[width=\linewidth]{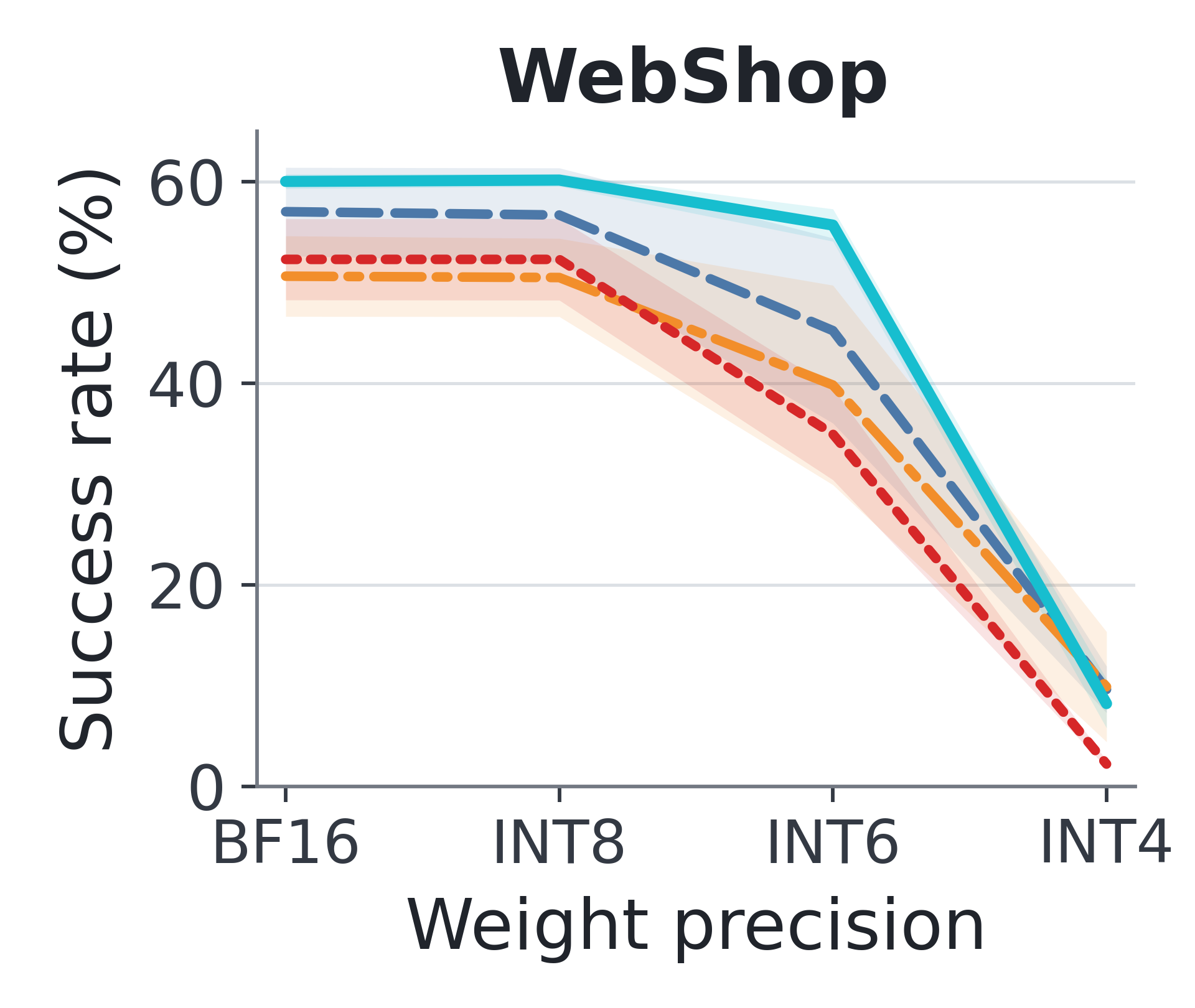}
    \caption{Weight quantization}\label{fig:main-webshop-quantization}
  \end{subfigure}%
 \\[1pt]
      \begin{subfigure}[t]{\MainPanelWidth}
    \includegraphics[width=\linewidth]{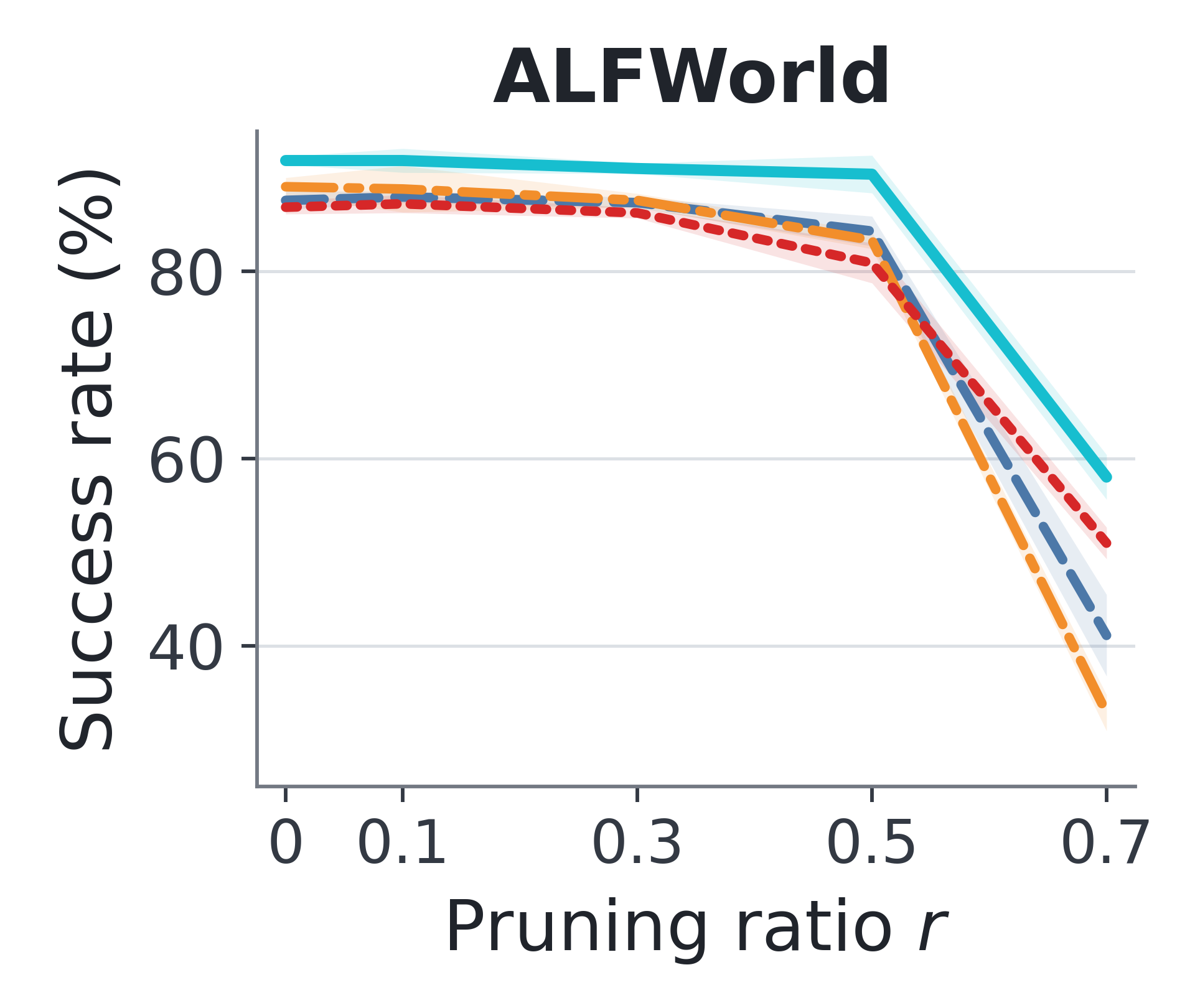}
    \caption{Unstructured FFN pruning}\label{fig:main-alfworld-unstructured-pruning}
  \end{subfigure}%
 &
      \begin{subfigure}[t]{\MainPanelWidth}
    \includegraphics[width=\linewidth]{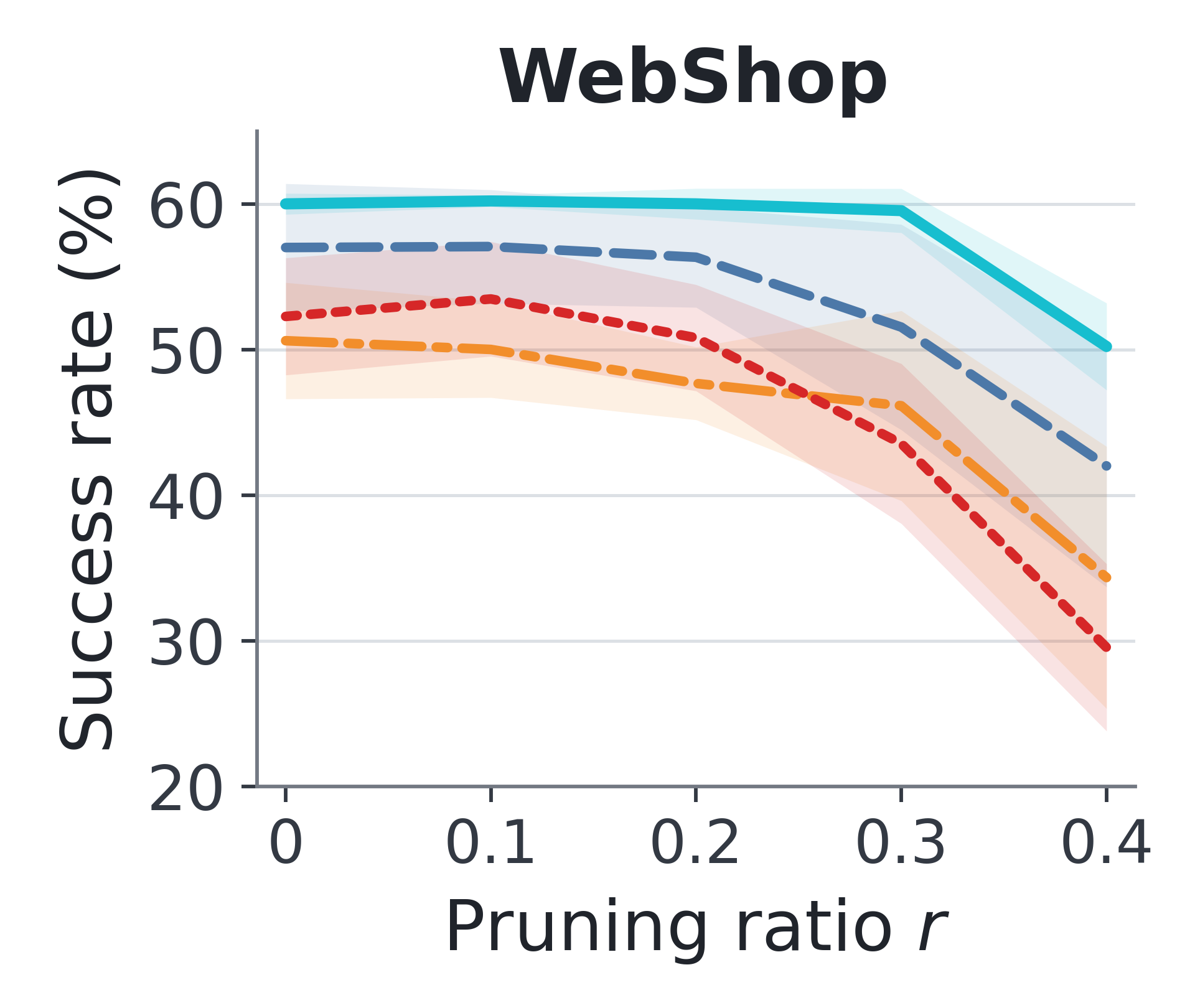}
    \caption{Unstructured FFN pruning}\label{fig:main-webshop-unstructured-pruning}
  \end{subfigure}%
 &
    \multicolumn{2}{@{}c@{}}{%
      \begin{minipage}[t]{\dimexpr2\MainPanelWidth+\MainPanelGap\relax}
        \includegraphics[width=\linewidth]{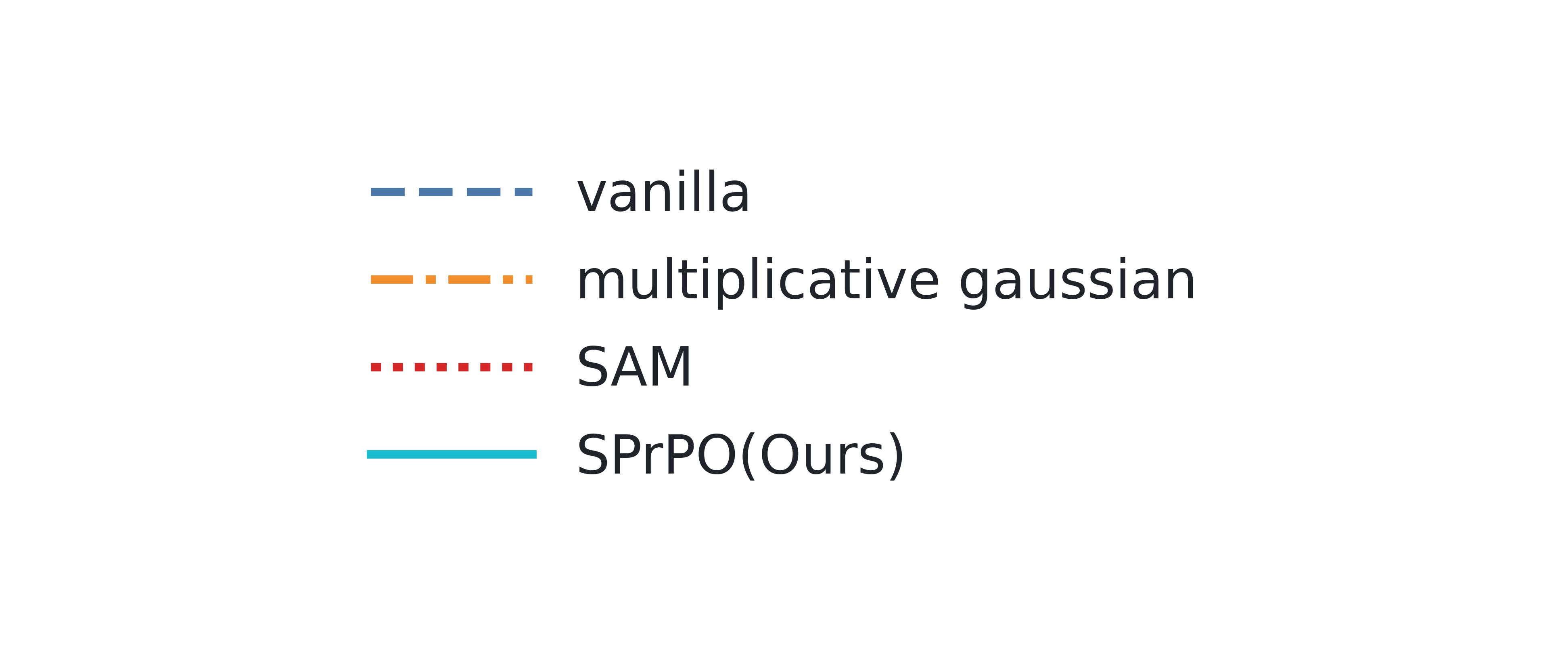}
      \end{minipage}%
    }
  \end{tabular}
  \caption{
  Evaluation of perturbation robustness on ALFWorld and WebShop across increasing perturbation magnitudes. The x-axis denotes perturbation strength, measured by Gaussian standard deviation $\sigma$, dropout probability $p$, pruning ratio $r$, and quantization bit-width.
  SPrPO generally exhibits less severe performance degradation and maintains higher success rates.
  }
  \label{fig:main-robustness}
\end{figure*}
\endgroup

%% file: images/main_results/tables/main_table.tex
\begingroup
\providecommand{\meanstd}[2]{\ensuremath{#1\,{\scriptstyle\pm\,#2}}}
\providecommand{\bestmeanstd}[2]{\ensuremath{\mathbf{#1}\,{\scriptstyle\pm\,#2}}}
\begin{table*}[t]
\centering
\small
\setlength{\tabcolsep}{6pt}
\renewcommand{\arraystretch}{1.16}
\caption{Evaluation of perturbation robustness on WebShop and ALFWorld. Results are reported as success rate$_{\mathrm{std}}$ (retention) over three seeds, where retention denotes the percentage of performance preserved relative to the corresponding clean policy. SPrPO achieves the strongest performance, maintaining both higher absolute success rates and higher relative retention.}
\label{tab:main-robustness}

\resizebox{\textwidth}{!}{%
\begin{tabular}{@{}lcccccc@{}}
\toprule
\textit{WebShop}
& Clean
& \shortstack{Gaussian\\($\sigma=.20$)}
& \shortstack{Dropout\\($p=.05$)}
& \shortstack{Structured Pruning\\($r=.20$)}
& \shortstack{Unstructured Pruning\\($r=.40$)}
& \shortstack{Quantization\\(INT6)} \\
\midrule
Vanilla
& \meanstd{57.00}{4.39}
& \meanstd{18.73}{3.03}~(32.9)
& \meanstd{12.20}{3.30}~(21.4)
& \meanstd{2.27}{1.86}~(4.0)
& \meanstd{42.00}{8.31}~(73.7)
& \meanstd{45.20}{9.18}~(79.3) \\
Gaussian
& \meanstd{50.60}{3.99}
& \meanstd{17.87}{5.12}~(35.3)
& \meanstd{11.33}{2.60}~(22.4)
& \meanstd{0.27}{0.12}~(0.5)
& \meanstd{34.33}{9.00}~(67.9)
& \meanstd{39.80}{9.90}~(78.7) \\
SAM
& \meanstd{52.27}{4.02}
& \meanstd{15.20}{2.65}~(29.1)
& \meanstd{8.87}{2.14}~(17.0)
& \meanstd{0.60}{0.87}~(1.1)
& \meanstd{29.53}{5.74}~(56.5)
& \meanstd{34.93}{4.54}~(66.8) \\
\rowcolor{blue!8}
\textbf{SPrPO (Ours)}
& \bestmeanstd{60.00}{0.72}
& \bestmeanstd{30.13}{1.81}~(50.2)
& \bestmeanstd{20.87}{2.00}~(34.8)
& \bestmeanstd{12.40}{3.65}~(20.7)
& \bestmeanstd{50.20}{2.99}~(83.7)
& \bestmeanstd{55.67}{1.60}~(92.8) \\
\midrule
\textit{ALFWorld}
& Clean
& \shortstack{Gaussian\\($\sigma=.20$)}
& \shortstack{Dropout\\($p=.05$)}
& \shortstack{Structured Pruning\\($r=.20$)}
& \shortstack{Unstructured Pruning\\($r=.70$)}
& \shortstack{Quantization\\(INT6)} \\
\midrule
Vanilla
& \meanstd{87.59}{0.97}
& \meanstd{64.96}{2.55}~(74.2)
& \meanstd{54.62}{1.65}~(62.4)
& \meanstd{41.73}{5.32}~(47.6)
& \meanstd{41.12}{4.35}~(46.9)
& \meanstd{86.98}{0.84}~(99.3) \\
Gaussian
& \meanstd{89.05}{0.97}
& \meanstd{72.26}{0.97}~(81.1)
& \meanstd{57.79}{1.80}~(64.9)
& \meanstd{23.72}{2.28}~(26.6)
& \meanstd{32.85}{1.93}~(36.9)
& \meanstd{87.47}{0.92}~(98.2) \\
SAM
& \meanstd{86.86}{0.73}
& \meanstd{66.79}{2.55}~(76.9)
& \meanstd{55.60}{0.76}~(64.0)
& \meanstd{22.02}{1.69}~(25.4)
& \meanstd{50.97}{1.69}~(58.7)
& \meanstd{86.25}{0.56}~(99.3) \\
\rowcolor{blue!8}
\textbf{SPrPO (Ours)}
& \bestmeanstd{91.85}{0.42}
& \bestmeanstd{75.91}{0.97}~(82.6)
& \bestmeanstd{64.60}{3.48}~(70.3)
& \bestmeanstd{47.08}{7.65}~(51.3)
& \bestmeanstd{58.03}{2.39}~(63.2)
& \bestmeanstd{91.24}{0.00}~(99.3) \\
\bottomrule
\end{tabular}%
}
\end{table*}
\endgroup

%% file: images/analysis/reward_surface/reward_surface.tex
\begingroup
\begin{figure*}[t]
  \centering
  \includegraphics[width=0.495\textwidth]{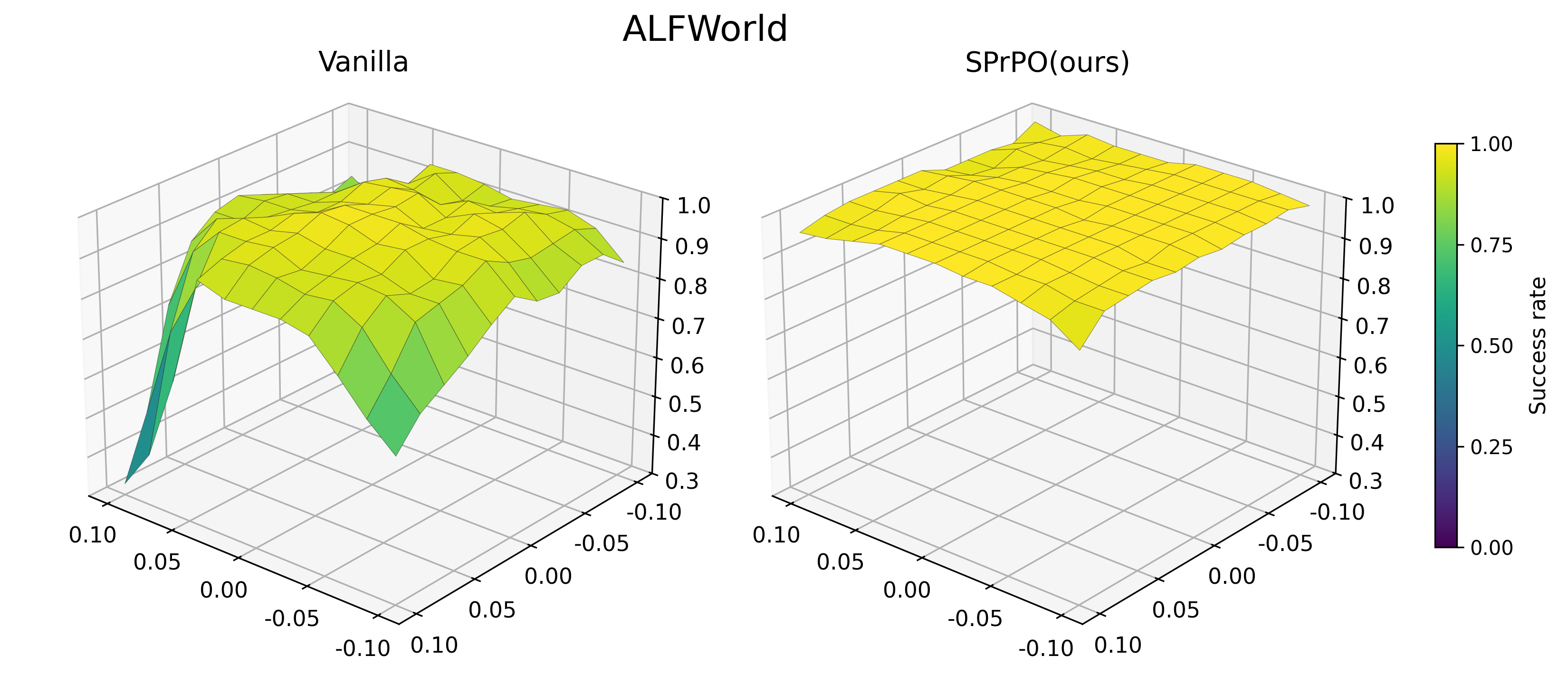}%
  \hfill
  \includegraphics[width=0.495\textwidth]{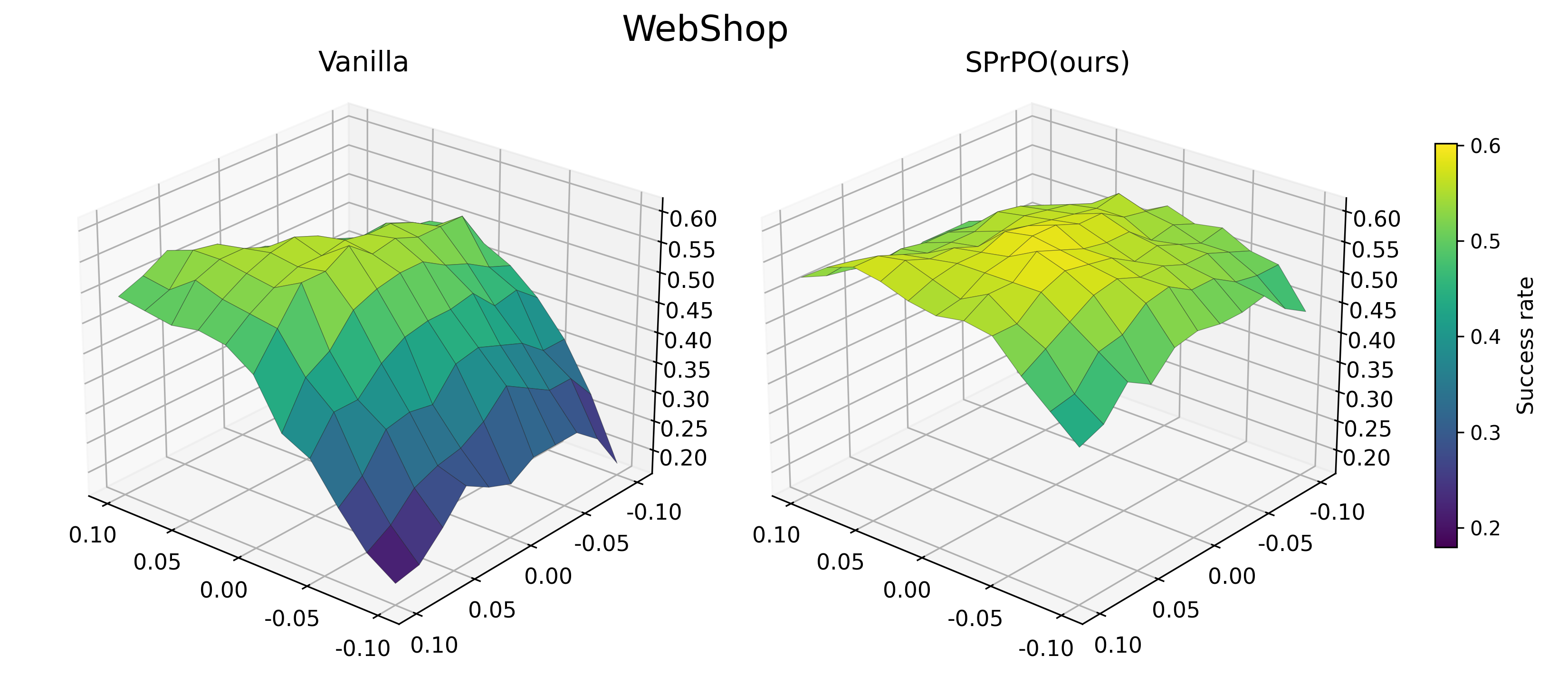}
  \caption{Reward landscape of policies trained with Vanilla and SPrPO on ALFWorld and WebShop. Vanilla exhibits sharp performance cliffs, while SPrPO shows flatter reward surfaces.}
  \label{fig:reward-surface}
\end{figure*}
\endgroup

%% file: sections/conclusion.tex
\section{Conclusion and Future Work}

We study perturbation robustness in RL-trained LLM agents and theoretically characterize policy improvement under perturbations. Building on this analysis, we develop SPrPO and show that it improves robustness across diverse policy perturbations. For future work, further analysis could consider different model structures, agent tasks, and perturbation settings.

%% file: statements/acknowledgment.tex
\section*{Acknowledgment}

This work used DeltaAI at the National Center for Supercomputing Applications (NCSA) through allocation CIS260333 from the Advanced Cyberinfrastructure Coordination Ecosystem: Services \& Support (ACCESS) program~\citep{boerner2023access}, which is supported by U.S. National Science Foundation grants \#2138259, \#2138286, \#2138307, \#2137603, and \#2138296.

%% file: statements/ai_use_statement.tex
\subsection*{AI use statement}
We used generative AI tools to assist with algorithm implementation and language polishing during the preparation of this work. All algorithmic design, theoretical analysis, experimental design, results, and conclusions were developed and verified by the authors. The authors take full responsibility for the content of the paper.

%% file: statements/reproducibility_statement.tex
\subsection*{Reproducibility statement}

Complete proofs of the theoretical results are included in Appendix~\ref{app:proofs}. Appendix~\ref{app:exp_details} provides detailed descriptions of the environments, training procedure, policy perturbations, algorithm implementation, and hyperparameters used in our experiments.

%% file: sections/appendix.tex
\input{appendix/detailed_proofs}

\input{appendix/more_related_works}
\input{appendix/experimental_details}

\input{appendix/more_experiment_results}
\input{appendix/text_examples}

%% file: appendix/detailed_proofs.tex
\section{Detailed Proofs for Theoretical Results}
\label{app:proofs}

\subsection{Proof of Theorem~\ref{thm:general_perturbation_improvement}} \label{app:proof_general_perturbation_improvement}

\begin{proof}
For each fixed perturbation $\xi$, applying Lemma~\ref{lem:trpo_improvement} to $\pi_\theta^\xi$ and $\pi_{\theta'}^\xi$ gives
\begin{equation*}
\Delta J_\xi
\ge
\frac{1}{1-\gamma}
\mathbb E_{\substack{
s\sim d_{\pi_\theta^\xi}\\
a\sim\pi_{\theta'}^\xi(\cdot\mid s)}}
\left[
A_{\pi_\theta^\xi}(s,a)
\right]
-
\frac{4\gamma A_{\max}}{(1-\gamma)^2}\alpha_\xi^2.
\end{equation*}

Taking expectation over $\xi\sim\nu$ yields
\begin{equation*}
\mathbb E_{\xi\sim\nu}[\Delta J_\xi]
\ge
\mathcal L_\nu(\theta')
-
\frac{4\gamma A_{\max}}{(1-\gamma)^2}
\mathbb E_{\xi\sim\nu}[\alpha_\xi^2]
=
m_\nu.
\end{equation*}

Let $\mu_\nu:=\mathbb E_{\xi\sim\nu}[\Delta J_\xi]$. Then $\mu_\nu\ge m_\nu$, and by assumption $\operatorname{Var}_\nu[\Delta J_\xi]\le v_\nu$. By Cantelli's inequality~\citep{Chebyshev}, for any $t>0$,
\begin{equation*}
\Pr_{\xi\sim\nu}
\left(
\Delta J_\xi-\mu_\nu\le -t
\right)
\le
\frac{\operatorname{Var}_\nu[\Delta J_\xi]}
{\operatorname{Var}_\nu[\Delta J_\xi]+t^2}.
\end{equation*}
Since $x/(x+t^2)$ is increasing in $x\ge 0$, choosing $t=\sqrt{(1-\beta)v_\nu/\beta}$ gives
\begin{equation*}
\Pr_{\xi\sim\nu}
\left(
\Delta J_\xi
<
\mu_\nu-\sqrt{\frac{1-\beta}{\beta}v_\nu}
\right)
\le
\frac{v_\nu}
{v_\nu+\frac{1-\beta}{\beta}v_\nu}
=
\beta.
\end{equation*}
Therefore,
\begin{equation*}
\Pr_{\xi\sim\nu}
\left(
\Delta J_\xi
\ge
m_\nu-\sqrt{\frac{1-\beta}{\beta}v_\nu}
\right)
\ge
\Pr_{\xi\sim\nu}
\left(
\Delta J_\xi
\ge
\mu_\nu-\sqrt{\frac{1-\beta}{\beta}v_\nu}
\right)
\ge
1-\beta.
\end{equation*}
\end{proof}

\subsection{Proof of Corollary~\ref{cor:gaussian_channel_improvement}}
\label{app:proof_gaussian_channel_improvement}

\begin{proof}
Let $C:=4\gamma A_{\max}/(1-\gamma)^2$. For each perturbation $\xi$, the TRPO improvement bound gives
\[
\Delta J_\xi
\ge
\mathcal L_\xi(\theta')-C\alpha_\xi^2
\ge
\mathcal L_\xi(\theta')-C\bar\alpha^2.
\]

Since $\xi\sim\mathcal N(0,\Sigma)$ and $\mathcal L_\xi(\theta')$ is continuously differentiable in $\xi$, the Gaussian Poincar\'e inequality~\citep{Gaussian_Poincare} gives the variance bound as follows:
\[
\begin{aligned}
\operatorname{Var}_{\xi\sim\nu}
[\mathcal L_\xi(\theta')]
&\le
\mathbb E_{\xi\sim\nu}
\left[
\nabla_\xi \mathcal L_\xi(\theta')^\top
\Sigma
\nabla_\xi \mathcal L_\xi(\theta')
\right] \\
&=
\sum_{i=1}^d \sigma_i^2 S_i(\theta') \\
&\le
\frac{\beta}{1-\beta}
\left(
\mathcal L_\nu(\theta')-C\bar\alpha^2
\right)^2.
\end{aligned}
\]

Applying Cantelli's inequality to $\mathcal L_\xi(\theta')$ yields
\[
\Pr_{\xi\sim\nu}
\left(
\mathcal L_\xi(\theta')-C\bar\alpha^2
\ge 0
\right)
\ge 1-\beta.
\]
Since $\Delta J_\xi\ge\mathcal L_\xi(\theta')-C\bar\alpha^2$ for every $\xi$, this yields $\Pr_{\xi\sim\nu}(\Delta J_\xi\ge0)\ge1-\beta$.
\end{proof}

\subsection{Proof of Theorem~\ref{thm:robust_policy_limit_points}}
\label{app:proof_robust_policy_limit_points}

\begin{proof}
\noindent\textbf{Step 1: Convergence.}
Since $\Pi$ is compact, every sequence in $\Pi$ admits a convergent
subsequence. Hence, $\{\pi_k\}_{k\ge0}$ has at least one limit point.
By construction, $\pi_{k+1}\in\mathcal U_\beta(\pi_k)$, which means
\[
\bar J_{k+1}-\bar J_k
\ge
m_{\nu,k}(\pi_{k+1})
\ge
\sqrt{
\frac{1-\beta}{\beta}
\operatorname{Var}_{\nu}
\left[
\Delta J(\pi_{k+1}^\xi;\pi_k^\xi)
\right]
}
\ge 0,
\]
where $\bar J_k:=\mathbb E_{\xi\sim\nu}[J(\pi_k^\xi)]$.
Thus, $\{\bar J_k\}_{k\ge0}$ is non-decreasing.
Since $J$ is bounded, $\{\bar J_k\}_{k\ge0}$ is also bounded above,
and therefore converges to a finite limit, denoted by $J^\star$.
Let $\bar\pi$ be any limit point of $\{\pi_k\}_{k\ge0}$, and let
$\pi_{k_j}\to\bar\pi$ be a convergent subsequence.
By continuity of the perturbed return and the dominated convergence
theorem, we have
\[
\mathbb E_{\xi\sim\nu}[J(\bar\pi^\xi)]
=
\lim_{j\to\infty}
\mathbb E_{\xi\sim\nu}[J(\pi_{k_j}^\xi)]
=
\lim_{j\to\infty}\bar J_{k_j}
=
J^\star.
\]

\noindent\textbf{Step 2: Probabilistic robustness.}
Let $\Xi_0\subseteq\Xi_\delta$ be the subset in
Assumption~\ref{ass:initial_variance}, and let
$\nu_0:=\nu(\cdot\mid\Xi_0)$.
Then $\nu(\Xi_0)\ge1-\beta_0$ and
$\sqrt{\operatorname{Var}_{\nu_0}[J(\pi_0^\xi)]}\le\sigma_0$.
For any square-integrable random variable $X$, conditioning on an
event of probability at least $1-\beta_0$ gives
$\operatorname{Var}_{\nu_0}[X]
\le(1-\beta_0)^{-1}\operatorname{Var}_{\nu}[X]$.
Together with the defining condition of $\mathcal U_\beta(\pi_k)$,
this yields
\[
\sqrt{
\operatorname{Var}_{\nu_0}
\left[
\Delta J(\pi_{k+1}^\xi;\pi_k^\xi)
\right]
}
\le
\frac{1}{\sqrt{1-\beta_0}}
\sqrt{\frac{\beta}{1-\beta}}
(\bar J_{k+1}-\bar J_k).
\]

By the triangle inequality for standard deviations,
\begin{align*}
\sqrt{
\operatorname{Var}_{\nu_0}[J(\pi_{k+1}^\xi)]
}
-
\sqrt{
\operatorname{Var}_{\nu_0}[J(\pi_k^\xi)]
}
&\le
\sqrt{
\operatorname{Var}_{\nu_0}
\left[
\Delta J(\pi_{k+1}^\xi;\pi_k^\xi)
\right]
} \\
&\le
\frac{1}{\sqrt{1-\beta_0}}
\sqrt{\frac{\beta}{1-\beta}}
(\bar J_{k+1}-\bar J_k).
\end{align*}

Summing from $0$ to $k-1$ gives
\begin{align*}
\sqrt{
\operatorname{Var}_{\nu_0}[J(\pi_k^\xi)]
}
&\le
\sigma_0
+
\frac{1}{\sqrt{1-\beta_0}}
\sqrt{\frac{\beta}{1-\beta}}
(\bar J_k-\bar J_0) \\
&\le
\sigma_0
+
\frac{C_J}{\sqrt{1-\beta_0}}
\sqrt{\frac{\beta}{1-\beta}}
=: \bar\sigma,
\end{align*}
where $C_J:=J_{\max}-J_{\min}$ denotes the range of the bounded
return.
By continuity and uniform boundedness of the perturbed return,
its first and second moments under $\nu_0$ converge along the
subsequence. Hence,
$\sqrt{\operatorname{Var}_{\nu_0}[J(\bar\pi^\xi)]}\le\bar\sigma$.
Since $\nu_0=\nu(\cdot\mid\Xi_0)$ with
$\nu(\Xi_0)\ge1-\beta_0$,
\[
\left|
\mathbb E_{\nu}[J(\bar\pi^\xi)]
-
\mathbb E_{\nu_0}[J(\bar\pi^\xi)]
\right|
\le
\beta_0 C_J.
\]

By Assumption~\ref{ass:initial_variance},
$\beta_0\le\bar\beta/2$ and
$\epsilon-\beta_0 C_J\ge\epsilon/2>0$.
Using the preceding bound on the conditional mean and applying
Cantelli's inequality, we obtain
\begin{align*}
\Pr_{\xi\sim\nu}
\left(
J(\bar\pi^\xi)<J^\star-\epsilon
\right)
&\le
\beta_0
+
(1-\beta_0)
\Pr_{\xi\sim\nu_0}
\left(
J(\bar\pi^\xi)
<
\mathbb E_{\xi\sim\nu_0}[J(\bar\pi^\xi)]
-
(\epsilon-\beta_0 C_J)
\right) \\
&\le
\beta_0
+
(1-\beta_0)
\frac{
\bar\sigma^2
}{
\bar\sigma^2+
(\epsilon-\beta_0 C_J)^2
}.
\end{align*}

The condition
$\beta\le\epsilon^2\bar\beta/(100C_J^2+\epsilon^2\bar\beta)$
implies
\[
C_J\sqrt{\frac{\beta}{1-\beta}}
\le
\frac{1}{10}\epsilon\sqrt{\bar\beta}.
\]
Together with
$\sigma_0<\epsilon\sqrt{\bar\beta}/4$
from Assumption~\ref{ass:initial_variance}, this gives
\begin{align*}
\beta_0
+
(1-\beta_0)
\frac{\bar\sigma^2}
{\bar\sigma^2+(\epsilon-\beta_0 C_J)^2}
&\le
\beta_0+
\frac{4(1-\beta_0)\bar\sigma^2}{\epsilon^2}
\\
&=
\beta_0+
\frac{4}{\epsilon^2}
\left(
\sqrt{1-\beta_0}\,\sigma_0
+
C_J\sqrt{\frac{\beta}{1-\beta}}
\right)^2
\\
&\le
\frac{\bar\beta}{2}
+
4\left(\frac14+\frac1{10}\right)^2\bar\beta
\\
&=
\frac{99}{100}\bar\beta
\le \bar\beta.
\end{align*}
It follows that
\[
\Pr_{\xi\sim\nu}
\left(
J(\bar\pi^\xi)\ge J^\star-\epsilon
\right)
\ge 1-\bar\beta,
\]
which shows that every limit point $\bar\pi$ is $(\delta,\epsilon,\bar\beta)$-probabilistic perturbation robust under $\nu_\delta$.
\end{proof}

%% file: appendix/more_related_works.tex
\section{Additional Related Work}
\label{app:more_related_works}

\paragraph{Robust Policy in Reinforcement Learning.}
There are several perturbation-based approaches to robust RL, and they have different designs for how perturbations are constructed and where they enter the interaction process. One major line constructs adversarial perturbations, including RARL~\citep{RARL}, Action Robust RL~\citep{ARRL}, and SHAPO~\citep{shapo}. These methods seek perturbations that minimize policy performance and train the policy against them, typically following a min--max objective $\max_{\theta}\min_{\xi\in\Xi} J(\pi_{\theta}^{\xi})$, where the adversarial perturbation may be approximated using gradient information, closely related to SAM-style~\citep{SAM} optimization.

Another line uses stochastic rather than worst-case perturbations, typically optimizing an average-case objective $\max_{\theta}\mathbb{E}_{\xi}[J(\pi_{\theta}^{\xi})]$. For example, \citet{policy_smoothing} inject Gaussian noise into observations, whereas \citet{param_noise} perturb policy parameters. Stochastic perturbations provide a simpler training mechanism, avoiding explicit inner optimization in constructing adversarial perturbations. Our work adopts stochastic perturbations and studies their role in improving the robustness of long-horizon LLM agents.

\paragraph{Perturbation-based Training for LLMs.}
Introducing perturbations during LLM training has been widely explored across different settings for better model robustness, including natural language understanding, instruction finetuning~\citep{neftune}, and reinforcement learning. Several studies consider adversarial perturbations. FreeLB~\citep{freelb} and PTP~\citep{ptp} apply adversarial perturbations in the input representation space in single-turn language tasks. Also, \citet{bilora} parameterize adversarial perturbations with a separate low-rank LoRA branch, which can be both more efficient and stable for LLMs.

For multi-turn or reasoning-based tasks, stochastic perturbations are more commonly used during finetuning for their simplicity and training stability. For example, \citet{noisyrollout} perturb model inputs during rollouts, and \citet{ncgrpo} perturb latent representations with Gaussian noise; both promote LLM exploration and diversify reasoning trajectories, thereby improving reasoning performance. More importantly, recent work shows that perturbation-based training can help models better preserve their performance after quantization. \citet{tamingweights} perturb sensitive components during supervised finetuning to better adapt quantization errors, while \citet{qerl} exploits and adaptively controls quantization-induced noise during RL to balance robustness and training stability. In contrast, we study perturbation-based training in long-horizon agentic interactions, where perturbations can alter subsequent observations over the trajectory. This motivates explicit consideration of optimization stability and more careful control of the injected perturbations.

%% file: appendix/experimental_details.tex
\section{Experiment Details}
\label{app:exp_details}

The code will be made publicly available upon acceptance.

\subsection{Environment Details}
\label{app:env_details}
\paragraph{Basic setup.} We conduct experiments on ALFWorld~\citep{alfworld} and WebShop~\citep{webshop}. ALFWorld is a text-based embodied environment for household tasks, with 3,553 tasks in its training split. We evaluate on all 274 validation tasks, consisting of 140 in-distribution and 134 out-of-distribution tasks. For WebShop, the benchmark contains 12,087 shopping instructions, with 500 tasks reserved for testing. We use the remaining 11,587 tasks as the RL
training pool and evaluate on all 500 test tasks. Episodes terminate upon task completion or when the maximum number of environment steps is reached, which is 50 for ALFWorld and 15 for WebShop.

\paragraph{Observation.}
At each environment step, the policy receives a textual observation containing the task instruction, recent interaction history, the current environment observation, and the currently admissible actions. We use the following common prompt structure for both environments:

\begin{promptbox}
\small
\textbf{Task:} \{task instruction\}

\smallskip
\textbf{Interaction context:} \{recent observations and actions\}

\smallskip
\textbf{Current observation:} \{current observation\}

\smallskip
\textbf{Available actions:} \{admissible actions\}

\smallskip
\textbf{Response:} Output the next action in \texttt{<action>...</action>}.
\end{promptbox}

\paragraph{Action.}
The policy directly generates an environment action at each step. For ALFWorld, the generated content is a structured natural-language command such as \texttt{go to fridge 1}; similarly for WebShop, it follows the environment action format, such as \texttt{search[wireless headphones]} or \texttt{click[Buy Now]}.

We use a no-thinking setting~\citep{no_thinking_in_RL1,no_thinking_in_RL2}, disabling explicit reasoning traces throughout both training and evaluation. This choice is motivated by both optimization and computational considerations. First, hallucinations in intermediate reasoning can propagate to the final decision. Second, RL updates over long reasoning traces can introduce token-level credit-assignment ambiguity, since incorrect intermediate thoughts may still be reinforced by a positive trajectory-level reward. Finally, thinking-free mode is also natural for many of the grounded tasks, where actions such as navigating to or picking up an object often do not require lengthy explicit deliberation. Besides, removing explicit reasoning substantially reduces training cost in our setting, shortening training from roughly seven days to two and yielding about a \textbf{3.5x} speedup.

\paragraph{Reward.}
Both environments provide sparse trajectory-level task rewards following GiGPO~\citep{gigpo}. ALFWorld uses a sparse success-based reward. WebShop instead provides a partial-completion score based on how well the purchased product matches the target attributes. We also apply a mild step-level format penalty of $-0.1$ when the response does not follow the required action format, which helps accelerate policy convergence.

\subsection{Training Details}
\label{app:train_details}

\paragraph{Base model and optimization.}
We fine-tune Qwen2.5-1.5B-Instruct with LoRA~\citep{lora} applied to all attention and FFN linear projections, keeping the base-model weights frozen. All methods use GiGPO~\citep{gigpo} for advantage estimation and optimize the same PPO clipped surrogate objective. All methods share the training hyperparameters in Table~\ref{tab:training-hparams}.

\paragraph{Perturbation implementation.}
Gaussian, SAM, and SPrPO perturb the intermediate representation in each FFN block in the same manner. Let $h_\ell=\operatorname{SiLU}(W_{\mathrm{gate},\ell}x_\ell)\odot W_{\mathrm{up},\ell}x_\ell$
denote the activated FFN representation at layer $\ell$ before the down-projection. We apply a channel-wise multiplicative perturbation
$h_\ell\leftarrow h_\ell\odot(1+\xi_\ell)$ before computing $W_{\mathrm{down},\ell}h_\ell$. The main perturbation hyperparameters are summarized in Table~\ref{tab:perturbation-hparams}. We also experimented with parameter-space noise and dropout during training, but both led to training collapse, consistent with our observation that RL-trained LLM agents are fragile to policy perturbations.

\begin{table}[t]
\centering
\small
\setlength{\tabcolsep}{5pt}
\caption{Shared training hyperparameters used across all methods.}
\label{tab:training-hparams}
\begin{tabular}{@{}lc@{}}
\toprule
Hyperparameter & Value \\
\midrule
Base model & Qwen2.5-1.5B-Instruct \\
LoRA rank / $\alpha_{\mathrm{LoRA}}$ & $64$ / $64$ \\
Optimizer & Adam \\
Learning rate & $3\times10^{-6}$ \\
Training steps & $200$ \\
Tasks per step & $16$ \\
Rollouts per task & $8$ \\
Trajectories per PPO batch & $128$ \\
Discount factor $\gamma$ & $0.95$ \\
PPO clip ratio & $0.2$ \\
Dual-clip coefficient & $3.0$ \\
Entropy coefficient & $0.001$ \\
WebShop prompt / response length & $4096$ / $512$ \\
ALFWorld prompt / response length & $3072$ / $1024$ \\
Validation temperature & $0.4$ \\
\bottomrule
\end{tabular}
\end{table}

\begin{table}[t]
\centering
\small
\setlength{\tabcolsep}{5pt}
\caption{Perturbation-specific training hyperparameters used in the main experiments.}
\label{tab:perturbation-hparams}
\begin{tabular}{@{}lll@{}}
\toprule
Method & Hyperparameter & Value \\
\midrule
Gaussian & Noise scale $\sigma$ & $0.10$ \\
SAM & Perturbation scale $\rho$ & $0.005$ \\
SPrPO & Base perturbation scale $\rho$ & $0.20$ \\
SPrPO & Sensitivity exponent $\alpha$ & $1.0$ \\
\bottomrule
\end{tabular}
\end{table}

\paragraph{Baseline 1: Vanilla.}
Vanilla uses standard GiGPO training without policy perturbations. The
resulting procedure is summarized in Algorithm~\ref{alg:vanilla}.

\begin{algorithm}[htbp]
\caption{Vanilla GiGPO}
\label{alg:vanilla}
\small
\begin{algorithmic}[1]
\Require Initial policy $\theta_1$, iterations $K$
\For{$k=1,\ldots,K$}
    \State Collect $\mathcal D_k\sim\pi_{\theta_k}$
    \State Estimate group-relative advantages $A_i$ for $i\in\mathcal D_k$
    \State Update the policy:
    $\theta_{k+1}\gets\theta_k+\eta_k\nabla_\theta\mathcal L_k(\theta_k)$
\EndFor
\State \Return $\theta_{K+1}$
\end{algorithmic}
\end{algorithm}

\paragraph{Baseline 2: Multiplicative Gaussian.}
The Gaussian baseline uses isotropic multiplicative noise; that is, the perturbation variance is the same across all perturbed FFN channels. At each training iteration, we sample a Gaussian perturbation and keep it fixed throughout the corresponding rollout and policy update.

\begin{algorithm}[htbp]
\caption{Multiplicative Gaussian perturbation training}
\label{alg:gaussian}
\small
\begin{algorithmic}[1]
\Require Initial policy $\theta_1$, iterations $K$, noise scale $\sigma$
\For{$k=1,\ldots,K$}
    \State Sample $z_k$, where $z_{k,\ell c}\sim\mathcal N(0,1)$
    \State Set $\xi_k\gets\sigma z_k$
    \State Collect $\mathcal D_k\sim\pi_{\theta_k}^{\xi_k}$
    and estimate advantages $A_i$
    \State Update the policy:
    $\theta_{k+1}\gets
    \theta_k+\eta_k\nabla_\theta\mathcal L_k(\theta_k;\xi_k)$
\EndFor
\State \Return $\theta_{K+1}$
\end{algorithmic}
\end{algorithm}

\paragraph{Baseline 3: SAM.}
We adapt Sharpness-Aware Minimization~\citep{SAM} to the same FFN-channel perturbation space used by Gaussian and SPrPO. At each training iteration, we perform one additional backward pass on the PPO loss to estimate the gradient with respect to multiplicative channel gain $g_{\ell c}=\left. \partial\mathcal L/\partial u_{\ell c}\right|_{u_{\ell c}=1}$. We normalize this gradient to construct an adversarial perturbation, then replay the same training batch under this perturbation to update the policy. SAM is extremely sensitive in our experiment: using $\rho=0.1$ leads to performance collapse. We therefore use a much smaller scale of $\rho=5\times10^{-3}$ to obtain reasonable performance.

\begin{algorithm}[htbp]
\caption{FFN-channel SAM perturbation training}
\label{alg:sam}
\small
\begin{algorithmic}[1]
\Require Initial policy $\theta_1$, iterations $K$, perturbation scale $\rho$
\For{$k=1,\ldots,K$}
    \State Collect $\mathcal D_k\sim\pi_{\theta_k}$
    and estimate advantages $A_i$
    \State Estimate channel gradients:
    $g_{k,\ell c}\gets
    \left.
    \dfrac{\partial\mathcal L_k(\theta_k;\mathbf u-\mathbf 1)}
    {\partial u_{\ell c}}
    \right|_{\mathbf u=\mathbf 1}$
    \State $\xi_k\gets-\rho\,g_k/\|g_k\|_2$
    \State Replay $\mathcal D_k$ under $\xi_k$ and update the policy:
    $\theta_{k+1}\gets
    \theta_k+\eta_k\nabla_\theta\mathcal L_k(\theta_k;\xi_k)$
\EndFor
\State \Return $\theta_{K+1}$
\end{algorithmic}
\end{algorithm}

\paragraph{Our method: SPrPO.}
We follow Algorithm~\ref{alg:prpo} in the main text. SPrPO estimates a
channel-wise sensitivity statistic from the gradient of the training
surrogate with respect to the multiplicative FFN perturbation. The resulting noise scale increases with the overall perturbation magnitude $\rho$ and decreases with channel sensitivity according to exponent $\alpha$, such that more sensitive channels receive smaller perturbations. After each policy update, we replay the current training batch at $\theta_{k+1}$ for one additional forward/backward pass to estimate the sensitivity statistics.

\paragraph{Training overhead.}
Table~\ref{tab:training-overhead} reports the observed training cost of SPrPO relative to Vanilla. Experiments were run with two NVIDIA RTX PRO 6000 Blackwell Max-Q GPUs on a host equipped with AMD EPYC
9354 CPU. SPrPO introduces no additional environment rollout, and its main computational overhead comes from one additional forward/backward pass that replays the current training batch after the policy update to estimate channel sensitivity.

\begin{table}[htbp]
\centering
\small
\setlength{\tabcolsep}{5pt}
\caption{Observed training cost of Vanilla and SPrPO on Qwen2.5-1.5B-Instruct. Times are averaged seconds per training step and exclude periodic validation and checkpoint saving.}
\label{tab:training-overhead}
\begin{tabular}{@{}llrrr@{}}
\toprule
Environment & Method & Rollout & Update & Total \\
\midrule
ALFWorld & Vanilla & 143.8 & 198.2 & 343.6 \\
ALFWorld & SPrPO & 171.5 & 400.1 & 572.9 \\
\midrule
WebShop & Vanilla & 50.1 & 60.0 & 110.6 \\
WebShop & SPrPO & 58.2 & 138.7 & 197.4 \\
\bottomrule
\end{tabular}
\end{table}

\subsection{Evaluation Details}
\label{app:eval_details}

\paragraph{General protocol.}
All evaluations use vLLM~\citep{vllm} inference, and use temperature $0.4$ for inference, following GiGPO~\citep{gigpo}. WebShop is evaluated on all $500$ goals, while ALFWorld uses the official validation split consisting of $140$ ID and $134$ OOD gamefiles. All perturbations below are applied only during evaluation to a fixed trained policy.

\paragraph{Channel-wise Gaussian perturbation.}
We use the same multiplicative channel-wise perturbation as during training. For the FFN activation $h_{\ell c}$ before the down-projection, we apply $\widetilde h_{\ell c}=(1+\xi_{\ell c})h_{\ell c}$, where $\xi_{\ell c}\sim\mathcal N(0,\sigma^2)$. We sweep $\sigma\in\{0.02,0.05,0.10,0.15,0.20\}$ on both WebShop and ALFWorld.

\paragraph{Channel dropout.}
We apply a Bernoulli mask to randomly remove FFN channels, independently dropping each channel with probability $p$. We sweep $p\in\{0.01,0.02,0.03,0.04,0.05\}$ on both WebShop and ALFWorld.

\paragraph{Structured pruning (Fisher importance).}
We prune FFN intermediate channels according to a diagonal-Fisher~\citep{fisher_pruning} importance score. For channel $c$ in layer $\ell$, we estimate $\widehat F_{\ell c}=\mathbb E_{\mathrm{calib}}[(h_{\ell c}\,\partial\mathcal L/\partial h_{\ell c})^2]$ using a fixed calibration set of $128$ agent-interaction sequences collected from a Vanilla-trained policy on the corresponding benchmark. All compared methods share this calibration set, while the Fisher scores are computed separately for each model using its own activations and gradients. Channels are ranked by $\widehat F_{\ell c}$, and the lowest-scoring fraction $r$ is removed by masking the corresponding rows of $W_{\mathrm{gate},\ell}$ and $W_{\mathrm{up},\ell}$ and columns of $W_{\mathrm{down},\ell}$. We sweep $r\in\{0.05,0.10,0.15,0.20\}$ on both WebShop and ALFWorld.

\paragraph{Unstructured pruning (SparseGPT).}
We use SparseGPT~\citep{sparseGPT} to prune a fraction $r$ of weights in each FFN projection matrix $W_{\mathrm{gate},\ell}$, $W_{\mathrm{up},\ell}$, and $W_{\mathrm{down},\ell}$, while leaving all attention weights dense. SparseGPT is applied sequentially across transformer layers using its Hessian-based weight selection and compensation procedure, with block size $128$ and Hessian damping $\lambda=0.01$. Calibration uses $128$ sequences drawn from the same rollout-derived calibration pool used for structured pruning. We sweep
$r\in\{0.10,0.20,0.30,0.40\}$ on WebShop and $r\in\{0.10,0.30,0.50,0.70\}$ on ALFWorld.

\paragraph{Quantization (round-to-nearest).}
We apply round-to-nearest (RTN) quantization to all linear weights in the transformer layers, including both attention and FFN projections. Weights are partitioned into groups of $128$, with each group sharing a quantization scale. For a group $\mathcal G$ and bit width $b$, we use $s_{\mathcal G}=\max_{w\in\mathcal G}|w|/(2^{b-1}-1)$. Each weight is divided by this scale, rounded to the nearest integer code, and clipped to the representable range. We evaluate INT8, INT6, and INT4 quantization.

%% file: appendix/more_experiment_results.tex
\section{More Experiment Results}
\label{app:more_experiment_results}

\paragraph{Main tables.}
We report the robustness results across all evaluated perturbation types and strengths. Results are mean $\pm$ standard deviation over three seeds, with boldface denoting the highest mean, including ties.

\begin{table*}[!htbp]
\centering
\small
\setlength{\tabcolsep}{4pt}
\renewcommand{\arraystretch}{1.10}
\caption{ALFWorld: overall success rate (\%).}
\label{tab:more-alfworld-combined-success-rate}
\resizebox{\linewidth}{!}{%
\begin{tabular}{@{}llcccc@{}}
\toprule
Perturbation & Strength & Vanilla & Gaussian & SAM & \textbf{SPrPO} \\
\midrule
Clean & BF16 & \ensuremath{87.59\,{\scriptstyle\pm\,0.97}} & \ensuremath{89.05\,{\scriptstyle\pm\,0.97}} & \ensuremath{86.86\,{\scriptstyle\pm\,0.73}} & \ensuremath{\mathbf{91.85}\,{\scriptstyle\pm\,0.42}} \\
\midrule
Gaussian & $\sigma=0.02$ & \ensuremath{88.44\,{\scriptstyle\pm\,1.38}} & \ensuremath{90.02\,{\scriptstyle\pm\,0.21}} & \ensuremath{86.01\,{\scriptstyle\pm\,0.76}} & \ensuremath{\mathbf{92.70}\,{\scriptstyle\pm\,0.36}} \\
 & $\sigma=0.05$ & \ensuremath{87.23\,{\scriptstyle\pm\,1.67}} & \ensuremath{89.54\,{\scriptstyle\pm\,1.28}} & \ensuremath{86.74\,{\scriptstyle\pm\,0.92}} & \ensuremath{\mathbf{91.12}\,{\scriptstyle\pm\,0.56}} \\
 & $\sigma=0.10$ & \ensuremath{84.79\,{\scriptstyle\pm\,2.23}} & \ensuremath{85.89\,{\scriptstyle\pm\,0.56}} & \ensuremath{86.13\,{\scriptstyle\pm\,0.73}} & \ensuremath{\mathbf{89.78}\,{\scriptstyle\pm\,0.63}} \\
 & $\sigma=0.15$ & \ensuremath{79.81\,{\scriptstyle\pm\,0.92}} & \ensuremath{81.51\,{\scriptstyle\pm\,1.11}} & \ensuremath{81.63\,{\scriptstyle\pm\,2.48}} & \ensuremath{\mathbf{86.13}\,{\scriptstyle\pm\,1.90}} \\
 & $\sigma=0.20$ & \ensuremath{64.96\,{\scriptstyle\pm\,2.55}} & \ensuremath{72.26\,{\scriptstyle\pm\,0.97}} & \ensuremath{66.79\,{\scriptstyle\pm\,2.55}} & \ensuremath{\mathbf{75.91}\,{\scriptstyle\pm\,0.97}} \\
\midrule
Dropout & $p=0.01$ & \ensuremath{82.97\,{\scriptstyle\pm\,2.43}} & \ensuremath{86.50\,{\scriptstyle\pm\,2.63}} & \ensuremath{85.16\,{\scriptstyle\pm\,1.65}} & \ensuremath{\mathbf{88.93}\,{\scriptstyle\pm\,0.21}} \\
 & $p=0.02$ & \ensuremath{79.81\,{\scriptstyle\pm\,0.84}} & \ensuremath{81.87\,{\scriptstyle\pm\,3.37}} & \ensuremath{81.51\,{\scriptstyle\pm\,1.17}} & \ensuremath{\mathbf{86.13}\,{\scriptstyle\pm\,0.97}} \\
 & $p=0.03$ & \ensuremath{71.53\,{\scriptstyle\pm\,1.32}} & \ensuremath{73.84\,{\scriptstyle\pm\,5.59}} & \ensuremath{74.33\,{\scriptstyle\pm\,3.60}} & \ensuremath{\mathbf{79.93}\,{\scriptstyle\pm\,1.93}} \\
 & $p=0.04$ & \ensuremath{63.50\,{\scriptstyle\pm\,1.67}} & \ensuremath{65.82\,{\scriptstyle\pm\,1.69}} & \ensuremath{63.14\,{\scriptstyle\pm\,2.55}} & \ensuremath{\mathbf{73.36}\,{\scriptstyle\pm\,2.19}} \\
 & $p=0.05$ & \ensuremath{54.62\,{\scriptstyle\pm\,1.65}} & \ensuremath{57.79\,{\scriptstyle\pm\,1.80}} & \ensuremath{55.60\,{\scriptstyle\pm\,0.76}} & \ensuremath{\mathbf{64.60}\,{\scriptstyle\pm\,3.48}} \\
\midrule
Structured pruning & $r=0.05$ & \ensuremath{83.82\,{\scriptstyle\pm\,1.38}} & \ensuremath{83.45\,{\scriptstyle\pm\,1.11}} & \ensuremath{85.64\,{\scriptstyle\pm\,1.72}} & \ensuremath{\mathbf{89.17}\,{\scriptstyle\pm\,0.56}} \\
 & $r=0.10$ & \ensuremath{72.38\,{\scriptstyle\pm\,2.23}} & \ensuremath{77.25\,{\scriptstyle\pm\,0.56}} & \ensuremath{72.99\,{\scriptstyle\pm\,0.36}} & \ensuremath{\mathbf{83.21}\,{\scriptstyle\pm\,0.73}} \\
 & $r=0.15$ & \ensuremath{59.61\,{\scriptstyle\pm\,2.01}} & \ensuremath{67.88\,{\scriptstyle\pm\,0.73}} & \ensuremath{49.76\,{\scriptstyle\pm\,4.59}} & \ensuremath{\mathbf{70.44}\,{\scriptstyle\pm\,2.03}} \\
 & $r=0.20$ & \ensuremath{41.73\,{\scriptstyle\pm\,5.32}} & \ensuremath{23.72\,{\scriptstyle\pm\,2.28}} & \ensuremath{22.02\,{\scriptstyle\pm\,1.69}} & \ensuremath{\mathbf{47.08}\,{\scriptstyle\pm\,7.65}} \\
\midrule
Unstructured pruning & $r=0.10$ & \ensuremath{87.96\,{\scriptstyle\pm\,0.63}} & \ensuremath{88.81\,{\scriptstyle\pm\,2.48}} & \ensuremath{87.23\,{\scriptstyle\pm\,0.97}} & \ensuremath{\mathbf{91.85}\,{\scriptstyle\pm\,1.28}} \\
 & $r=0.30$ & \ensuremath{87.35\,{\scriptstyle\pm\,0.56}} & \ensuremath{87.59\,{\scriptstyle\pm\,0.73}} & \ensuremath{86.25\,{\scriptstyle\pm\,0.56}} & \ensuremath{\mathbf{91.00}\,{\scriptstyle\pm\,0.56}} \\
 & $r=0.50$ & \ensuremath{84.31\,{\scriptstyle\pm\,1.59}} & \ensuremath{83.33\,{\scriptstyle\pm\,0.92}} & \ensuremath{80.90\,{\scriptstyle\pm\,2.14}} & \ensuremath{\mathbf{90.39}\,{\scriptstyle\pm\,2.01}} \\
 & $r=0.70$ & \ensuremath{41.12\,{\scriptstyle\pm\,4.35}} & \ensuremath{32.85\,{\scriptstyle\pm\,1.93}} & \ensuremath{50.97\,{\scriptstyle\pm\,1.69}} & \ensuremath{\mathbf{58.03}\,{\scriptstyle\pm\,2.39}} \\
\midrule
Quantization & INT8 & \ensuremath{88.44\,{\scriptstyle\pm\,0.92}} & \ensuremath{88.81\,{\scriptstyle\pm\,0.84}} & \ensuremath{86.86\,{\scriptstyle\pm\,0.97}} & \ensuremath{\mathbf{91.24}\,{\scriptstyle\pm\,0.73}} \\
 & INT6 & \ensuremath{86.98\,{\scriptstyle\pm\,0.84}} & \ensuremath{87.47\,{\scriptstyle\pm\,0.92}} & \ensuremath{86.25\,{\scriptstyle\pm\,0.56}} & \ensuremath{\mathbf{91.24}\,{\scriptstyle\pm\,0.00}} \\
 & INT4 & \ensuremath{47.81\,{\scriptstyle\pm\,3.18}} & \ensuremath{39.17\,{\scriptstyle\pm\,4.04}} & \ensuremath{\mathbf{59.61}\,{\scriptstyle\pm\,2.56}} & \ensuremath{48.05\,{\scriptstyle\pm\,4.02}} \\
\bottomrule
\end{tabular}%
}
\end{table*}

\begin{table*}[!htbp]
\centering
\small
\setlength{\tabcolsep}{4pt}
\renewcommand{\arraystretch}{1.10}
\caption{ALFWorld: ID and OOD success rates (\%).}
\label{tab:more-alfworld-id-ood-success-rate}
\resizebox{\linewidth}{!}{%
\begin{tabular}{@{}llcccccccc@{}}
\toprule
\multicolumn{2}{c}{} & \multicolumn{4}{c}{ID} & \multicolumn{4}{c}{OOD} \\
\cmidrule(lr){3-6}\cmidrule(lr){7-10}
Perturbation & Strength & Vanilla & Gaussian & SAM & \textbf{SPrPO} & Vanilla & Gaussian & SAM & \textbf{SPrPO} \\
\midrule
Clean & BF16 & \ensuremath{89.05\,{\scriptstyle\pm\,1.09}} & \ensuremath{89.29\,{\scriptstyle\pm\,1.24}} & \ensuremath{90.48\,{\scriptstyle\pm\,2.06}} & \ensuremath{\mathbf{94.05}\,{\scriptstyle\pm\,0.82}} & \ensuremath{86.07\,{\scriptstyle\pm\,1.55}} & \ensuremath{88.81\,{\scriptstyle\pm\,1.97}} & \ensuremath{83.08\,{\scriptstyle\pm\,2.62}} & \ensuremath{\mathbf{89.55}\,{\scriptstyle\pm\,0.00}} \\
\midrule
Gaussian & $\sigma=0.02$ & \ensuremath{89.52\,{\scriptstyle\pm\,2.70}} & \ensuremath{90.24\,{\scriptstyle\pm\,0.82}} & \ensuremath{88.33\,{\scriptstyle\pm\,0.41}} & \ensuremath{\mathbf{93.33}\,{\scriptstyle\pm\,1.65}} & \ensuremath{87.31\,{\scriptstyle\pm\,0.75}} & \ensuremath{89.80\,{\scriptstyle\pm\,0.43}} & \ensuremath{83.58\,{\scriptstyle\pm\,1.49}} & \ensuremath{\mathbf{92.04}\,{\scriptstyle\pm\,1.14}} \\
 & $\sigma=0.05$ & \ensuremath{86.90\,{\scriptstyle\pm\,1.80}} & \ensuremath{90.48\,{\scriptstyle\pm\,1.49}} & \ensuremath{89.29\,{\scriptstyle\pm\,1.43}} & \ensuremath{\mathbf{92.14}\,{\scriptstyle\pm\,1.43}} & \ensuremath{87.56\,{\scriptstyle\pm\,1.55}} & \ensuremath{88.56\,{\scriptstyle\pm\,1.14}} & \ensuremath{84.08\,{\scriptstyle\pm\,2.83}} & \ensuremath{\mathbf{90.05}\,{\scriptstyle\pm\,0.43}} \\
 & $\sigma=0.10$ & \ensuremath{86.90\,{\scriptstyle\pm\,2.89}} & \ensuremath{85.95\,{\scriptstyle\pm\,1.49}} & \ensuremath{88.33\,{\scriptstyle\pm\,1.80}} & \ensuremath{\mathbf{90.48}\,{\scriptstyle\pm\,2.30}} & \ensuremath{82.59\,{\scriptstyle\pm\,3.02}} & \ensuremath{85.82\,{\scriptstyle\pm\,1.49}} & \ensuremath{83.83\,{\scriptstyle\pm\,0.43}} & \ensuremath{\mathbf{89.05}\,{\scriptstyle\pm\,2.28}} \\
 & $\sigma=0.15$ & \ensuremath{77.62\,{\scriptstyle\pm\,1.65}} & \ensuremath{82.86\,{\scriptstyle\pm\,2.14}} & \ensuremath{\mathbf{85.71}\,{\scriptstyle\pm\,3.11}} & \ensuremath{\mathbf{85.71}\,{\scriptstyle\pm\,2.14}} & \ensuremath{82.09\,{\scriptstyle\pm\,0.75}} & \ensuremath{80.10\,{\scriptstyle\pm\,1.88}} & \ensuremath{77.36\,{\scriptstyle\pm\,2.28}} & \ensuremath{\mathbf{86.57}\,{\scriptstyle\pm\,4.48}} \\
 & $\sigma=0.20$ & \ensuremath{65.48\,{\scriptstyle\pm\,2.89}} & \ensuremath{71.67\,{\scriptstyle\pm\,1.09}} & \ensuremath{67.14\,{\scriptstyle\pm\,3.11}} & \ensuremath{\mathbf{76.90}\,{\scriptstyle\pm\,1.65}} & \ensuremath{64.43\,{\scriptstyle\pm\,3.02}} & \ensuremath{72.89\,{\scriptstyle\pm\,1.55}} & \ensuremath{66.42\,{\scriptstyle\pm\,2.59}} & \ensuremath{\mathbf{74.88}\,{\scriptstyle\pm\,0.43}} \\
\midrule
Dropout & $p=0.01$ & \ensuremath{84.76\,{\scriptstyle\pm\,1.80}} & \ensuremath{87.86\,{\scriptstyle\pm\,1.89}} & \ensuremath{88.33\,{\scriptstyle\pm\,1.09}} & \ensuremath{\mathbf{89.29}\,{\scriptstyle\pm\,1.89}} & \ensuremath{81.09\,{\scriptstyle\pm\,3.11}} & \ensuremath{85.07\,{\scriptstyle\pm\,3.88}} & \ensuremath{81.84\,{\scriptstyle\pm\,3.11}} & \ensuremath{\mathbf{88.56}\,{\scriptstyle\pm\,1.55}} \\
 & $p=0.02$ & \ensuremath{79.05\,{\scriptstyle\pm\,0.41}} & \ensuremath{82.62\,{\scriptstyle\pm\,2.70}} & \ensuremath{83.57\,{\scriptstyle\pm\,0.71}} & \ensuremath{\mathbf{87.14}\,{\scriptstyle\pm\,0.71}} & \ensuremath{80.60\,{\scriptstyle\pm\,1.29}} & \ensuremath{81.09\,{\scriptstyle\pm\,4.11}} & \ensuremath{79.35\,{\scriptstyle\pm\,2.62}} & \ensuremath{\mathbf{85.07}\,{\scriptstyle\pm\,1.49}} \\
 & $p=0.03$ & \ensuremath{73.10\,{\scriptstyle\pm\,1.49}} & \ensuremath{75.48\,{\scriptstyle\pm\,5.07}} & \ensuremath{75.24\,{\scriptstyle\pm\,2.51}} & \ensuremath{\mathbf{80.24}\,{\scriptstyle\pm\,1.49}} & \ensuremath{69.90\,{\scriptstyle\pm\,1.14}} & \ensuremath{72.14\,{\scriptstyle\pm\,8.22}} & \ensuremath{73.38\,{\scriptstyle\pm\,4.97}} & \ensuremath{\mathbf{79.60}\,{\scriptstyle\pm\,2.83}} \\
 & $p=0.04$ & \ensuremath{63.81\,{\scriptstyle\pm\,2.70}} & \ensuremath{68.81\,{\scriptstyle\pm\,3.38}} & \ensuremath{62.86\,{\scriptstyle\pm\,1.43}} & \ensuremath{\mathbf{73.10}\,{\scriptstyle\pm\,4.06}} & \ensuremath{63.18\,{\scriptstyle\pm\,4.56}} & \ensuremath{62.69\,{\scriptstyle\pm\,0.75}} & \ensuremath{63.43\,{\scriptstyle\pm\,3.73}} & \ensuremath{\mathbf{73.63}\,{\scriptstyle\pm\,1.14}} \\
 & $p=0.05$ & \ensuremath{52.62\,{\scriptstyle\pm\,1.09}} & \ensuremath{58.33\,{\scriptstyle\pm\,3.67}} & \ensuremath{55.24\,{\scriptstyle\pm\,2.30}} & \ensuremath{\mathbf{65.00}\,{\scriptstyle\pm\,4.34}} & \ensuremath{56.72\,{\scriptstyle\pm\,3.95}} & \ensuremath{57.21\,{\scriptstyle\pm\,2.40}} & \ensuremath{55.97\,{\scriptstyle\pm\,2.24}} & \ensuremath{\mathbf{64.18}\,{\scriptstyle\pm\,3.95}} \\
\midrule
Structured pruning & $r=0.05$ & \ensuremath{86.67\,{\scriptstyle\pm\,0.82}} & \ensuremath{85.24\,{\scriptstyle\pm\,2.51}} & \ensuremath{87.86\,{\scriptstyle\pm\,0.71}} & \ensuremath{\mathbf{92.86}\,{\scriptstyle\pm\,0.71}} & \ensuremath{80.85\,{\scriptstyle\pm\,2.40}} & \ensuremath{81.59\,{\scriptstyle\pm\,4.37}} & \ensuremath{83.33\,{\scriptstyle\pm\,2.83}} & \ensuremath{\mathbf{85.32}\,{\scriptstyle\pm\,1.55}} \\
 & $r=0.10$ & \ensuremath{72.38\,{\scriptstyle\pm\,3.38}} & \ensuremath{77.86\,{\scriptstyle\pm\,0.71}} & \ensuremath{72.62\,{\scriptstyle\pm\,1.09}} & \ensuremath{\mathbf{83.57}\,{\scriptstyle\pm\,0.71}} & \ensuremath{72.39\,{\scriptstyle\pm\,3.73}} & \ensuremath{76.62\,{\scriptstyle\pm\,1.72}} & \ensuremath{73.38\,{\scriptstyle\pm\,1.14}} & \ensuremath{\mathbf{82.84}\,{\scriptstyle\pm\,1.97}} \\
 & $r=0.15$ & \ensuremath{56.90\,{\scriptstyle\pm\,3.67}} & \ensuremath{65.95\,{\scriptstyle\pm\,1.49}} & \ensuremath{48.81\,{\scriptstyle\pm\,2.51}} & \ensuremath{\mathbf{66.90}\,{\scriptstyle\pm\,3.22}} & \ensuremath{62.44\,{\scriptstyle\pm\,3.37}} & \ensuremath{69.90\,{\scriptstyle\pm\,2.40}} & \ensuremath{50.75\,{\scriptstyle\pm\,6.84}} & \ensuremath{\mathbf{74.13}\,{\scriptstyle\pm\,1.14}} \\
 & $r=0.20$ & \ensuremath{36.19\,{\scriptstyle\pm\,3.93}} & \ensuremath{21.43\,{\scriptstyle\pm\,2.47}} & \ensuremath{16.43\,{\scriptstyle\pm\,2.47}} & \ensuremath{\mathbf{40.71}\,{\scriptstyle\pm\,5.58}} & \ensuremath{47.51\,{\scriptstyle\pm\,6.77}} & \ensuremath{26.12\,{\scriptstyle\pm\,3.42}} & \ensuremath{27.86\,{\scriptstyle\pm\,0.86}} & \ensuremath{\mathbf{53.73}\,{\scriptstyle\pm\,10.10}} \\
\midrule
Unstructured pruning & $r=0.10$ & \ensuremath{88.33\,{\scriptstyle\pm\,1.09}} & \ensuremath{89.76\,{\scriptstyle\pm\,2.97}} & \ensuremath{90.48\,{\scriptstyle\pm\,2.89}} & \ensuremath{\mathbf{92.38}\,{\scriptstyle\pm\,0.82}} & \ensuremath{87.56\,{\scriptstyle\pm\,0.43}} & \ensuremath{87.81\,{\scriptstyle\pm\,3.02}} & \ensuremath{83.83\,{\scriptstyle\pm\,1.14}} & \ensuremath{\mathbf{91.29}\,{\scriptstyle\pm\,1.88}} \\
 & $r=0.30$ & \ensuremath{87.14\,{\scriptstyle\pm\,0.71}} & \ensuremath{87.38\,{\scriptstyle\pm\,1.49}} & \ensuremath{89.05\,{\scriptstyle\pm\,0.41}} & \ensuremath{\mathbf{92.86}\,{\scriptstyle\pm\,1.89}} & \ensuremath{87.56\,{\scriptstyle\pm\,0.43}} & \ensuremath{87.81\,{\scriptstyle\pm\,2.83}} & \ensuremath{83.33\,{\scriptstyle\pm\,1.14}} & \ensuremath{\mathbf{89.05}\,{\scriptstyle\pm\,0.86}} \\
 & $r=0.50$ & \ensuremath{84.29\,{\scriptstyle\pm\,2.58}} & \ensuremath{84.76\,{\scriptstyle\pm\,1.09}} & \ensuremath{84.29\,{\scriptstyle\pm\,2.58}} & \ensuremath{\mathbf{91.90}\,{\scriptstyle\pm\,2.18}} & \ensuremath{84.33\,{\scriptstyle\pm\,0.75}} & \ensuremath{81.84\,{\scriptstyle\pm\,0.86}} & \ensuremath{77.36\,{\scriptstyle\pm\,1.72}} & \ensuremath{\mathbf{88.81}\,{\scriptstyle\pm\,1.97}} \\
 & $r=0.70$ & \ensuremath{41.43\,{\scriptstyle\pm\,4.68}} & \ensuremath{31.43\,{\scriptstyle\pm\,1.24}} & \ensuremath{51.90\,{\scriptstyle\pm\,1.80}} & \ensuremath{\mathbf{53.81}\,{\scriptstyle\pm\,3.52}} & \ensuremath{40.80\,{\scriptstyle\pm\,4.80}} & \ensuremath{34.33\,{\scriptstyle\pm\,3.25}} & \ensuremath{50.00\,{\scriptstyle\pm\,3.73}} & \ensuremath{\mathbf{62.44}\,{\scriptstyle\pm\,2.62}} \\
\midrule
Quantization & INT8 & \ensuremath{89.29\,{\scriptstyle\pm\,1.24}} & \ensuremath{89.76\,{\scriptstyle\pm\,1.80}} & \ensuremath{90.24\,{\scriptstyle\pm\,1.65}} & \ensuremath{\mathbf{91.90}\,{\scriptstyle\pm\,0.82}} & \ensuremath{87.56\,{\scriptstyle\pm\,0.86}} & \ensuremath{87.81\,{\scriptstyle\pm\,1.14}} & \ensuremath{83.33\,{\scriptstyle\pm\,3.02}} & \ensuremath{\mathbf{90.55}\,{\scriptstyle\pm\,1.72}} \\
 & INT6 & \ensuremath{87.38\,{\scriptstyle\pm\,0.82}} & \ensuremath{88.81\,{\scriptstyle\pm\,2.06}} & \ensuremath{89.05\,{\scriptstyle\pm\,1.09}} & \ensuremath{\mathbf{91.67}\,{\scriptstyle\pm\,1.49}} & \ensuremath{86.57\,{\scriptstyle\pm\,1.49}} & \ensuremath{86.07\,{\scriptstyle\pm\,3.02}} & \ensuremath{83.33\,{\scriptstyle\pm\,0.43}} & \ensuremath{\mathbf{90.80}\,{\scriptstyle\pm\,1.55}} \\
 & INT4 & \ensuremath{47.62\,{\scriptstyle\pm\,2.06}} & \ensuremath{42.62\,{\scriptstyle\pm\,5.46}} & \ensuremath{\mathbf{57.86}\,{\scriptstyle\pm\,4.95}} & \ensuremath{49.76\,{\scriptstyle\pm\,6.64}} & \ensuremath{48.01\,{\scriptstyle\pm\,4.37}} & \ensuremath{35.57\,{\scriptstyle\pm\,2.62}} & \ensuremath{\mathbf{61.44}\,{\scriptstyle\pm\,2.28}} & \ensuremath{46.27\,{\scriptstyle\pm\,1.29}} \\
\bottomrule
\end{tabular}%
}
\end{table*}
\clearpage

\begin{table*}[!htbp]
\centering
\small
\setlength{\tabcolsep}{4pt}
\renewcommand{\arraystretch}{1.10}
\caption{WebShop: success rate (\%).}
\label{tab:more-webshop-all-success-rate}
\resizebox{\linewidth}{!}{%
\begin{tabular}{@{}llcccc@{}}
\toprule
Perturbation & Strength & Vanilla & Gaussian & SAM & \textbf{SPrPO} \\
\midrule
Clean & BF16 & \ensuremath{57.00\,{\scriptstyle\pm\,4.39}} & \ensuremath{50.60\,{\scriptstyle\pm\,3.99}} & \ensuremath{52.27\,{\scriptstyle\pm\,4.02}} & \ensuremath{\mathbf{60.00}\,{\scriptstyle\pm\,0.72}} \\
\midrule
Gaussian & $\sigma=0.02$ & \ensuremath{57.00\,{\scriptstyle\pm\,3.82}} & \ensuremath{51.07\,{\scriptstyle\pm\,3.35}} & \ensuremath{52.13\,{\scriptstyle\pm\,2.95}} & \ensuremath{\mathbf{59.73}\,{\scriptstyle\pm\,1.55}} \\
 & $\sigma=0.05$ & \ensuremath{53.47\,{\scriptstyle\pm\,4.88}} & \ensuremath{48.47\,{\scriptstyle\pm\,4.80}} & \ensuremath{49.00\,{\scriptstyle\pm\,2.31}} & \ensuremath{\mathbf{58.67}\,{\scriptstyle\pm\,1.29}} \\
 & $\sigma=0.10$ & \ensuremath{45.87\,{\scriptstyle\pm\,6.65}} & \ensuremath{41.53\,{\scriptstyle\pm\,5.35}} & \ensuremath{40.00\,{\scriptstyle\pm\,4.13}} & \ensuremath{\mathbf{53.07}\,{\scriptstyle\pm\,1.29}} \\
 & $\sigma=0.15$ & \ensuremath{32.73\,{\scriptstyle\pm\,2.95}} & \ensuremath{30.33\,{\scriptstyle\pm\,7.83}} & \ensuremath{27.27\,{\scriptstyle\pm\,3.69}} & \ensuremath{\mathbf{44.33}\,{\scriptstyle\pm\,0.50}} \\
 & $\sigma=0.20$ & \ensuremath{18.73\,{\scriptstyle\pm\,3.03}} & \ensuremath{17.87\,{\scriptstyle\pm\,5.12}} & \ensuremath{15.20\,{\scriptstyle\pm\,2.65}} & \ensuremath{\mathbf{30.13}\,{\scriptstyle\pm\,1.81}} \\
\midrule
Dropout & $p=0.01$ & \ensuremath{44.33\,{\scriptstyle\pm\,5.83}} & \ensuremath{40.67\,{\scriptstyle\pm\,6.30}} & \ensuremath{38.67\,{\scriptstyle\pm\,4.39}} & \ensuremath{\mathbf{54.87}\,{\scriptstyle\pm\,0.42}} \\
 & $p=0.02$ & \ensuremath{35.20\,{\scriptstyle\pm\,4.42}} & \ensuremath{31.80\,{\scriptstyle\pm\,6.93}} & \ensuremath{29.73\,{\scriptstyle\pm\,5.35}} & \ensuremath{\mathbf{46.13}\,{\scriptstyle\pm\,1.81}} \\
 & $p=0.03$ & \ensuremath{25.53\,{\scriptstyle\pm\,5.25}} & \ensuremath{24.07\,{\scriptstyle\pm\,6.20}} & \ensuremath{20.80\,{\scriptstyle\pm\,3.82}} & \ensuremath{\mathbf{38.40}\,{\scriptstyle\pm\,2.80}} \\
 & $p=0.04$ & \ensuremath{18.60\,{\scriptstyle\pm\,2.95}} & \ensuremath{16.27\,{\scriptstyle\pm\,4.91}} & \ensuremath{13.87\,{\scriptstyle\pm\,2.91}} & \ensuremath{\mathbf{29.47}\,{\scriptstyle\pm\,2.58}} \\
 & $p=0.05$ & \ensuremath{12.20\,{\scriptstyle\pm\,3.30}} & \ensuremath{11.33\,{\scriptstyle\pm\,2.60}} & \ensuremath{8.87\,{\scriptstyle\pm\,2.14}} & \ensuremath{\mathbf{20.87}\,{\scriptstyle\pm\,2.00}} \\
\midrule
Structured pruning & $r=0.05$ & \ensuremath{49.93\,{\scriptstyle\pm\,8.60}} & \ensuremath{47.80\,{\scriptstyle\pm\,3.17}} & \ensuremath{45.33\,{\scriptstyle\pm\,2.80}} & \ensuremath{\mathbf{55.80}\,{\scriptstyle\pm\,3.52}} \\
 & $r=0.10$ & \ensuremath{30.13\,{\scriptstyle\pm\,7.80}} & \ensuremath{27.40\,{\scriptstyle\pm\,10.51}} & \ensuremath{20.20\,{\scriptstyle\pm\,5.37}} & \ensuremath{\mathbf{43.87}\,{\scriptstyle\pm\,4.15}} \\
 & $r=0.15$ & \ensuremath{6.33\,{\scriptstyle\pm\,4.05}} & \ensuremath{18.67\,{\scriptstyle\pm\,7.20}} & \ensuremath{6.60\,{\scriptstyle\pm\,2.75}} & \ensuremath{\mathbf{23.27}\,{\scriptstyle\pm\,2.58}} \\
 & $r=0.20$ & \ensuremath{2.27\,{\scriptstyle\pm\,1.86}} & \ensuremath{0.27\,{\scriptstyle\pm\,0.12}} & \ensuremath{0.60\,{\scriptstyle\pm\,0.87}} & \ensuremath{\mathbf{12.40}\,{\scriptstyle\pm\,3.65}} \\
\midrule
Unstructured pruning & $r=0.10$ & \ensuremath{57.07\,{\scriptstyle\pm\,3.90}} & \ensuremath{50.00\,{\scriptstyle\pm\,3.30}} & \ensuremath{53.47\,{\scriptstyle\pm\,3.93}} & \ensuremath{\mathbf{60.20}\,{\scriptstyle\pm\,0.40}} \\
 & $r=0.20$ & \ensuremath{56.33\,{\scriptstyle\pm\,3.44}} & \ensuremath{47.67\,{\scriptstyle\pm\,2.50}} & \ensuremath{50.80\,{\scriptstyle\pm\,3.65}} & \ensuremath{\mathbf{60.00}\,{\scriptstyle\pm\,1.06}} \\
 & $r=0.30$ & \ensuremath{51.53\,{\scriptstyle\pm\,7.06}} & \ensuremath{46.13\,{\scriptstyle\pm\,6.52}} & \ensuremath{43.53\,{\scriptstyle\pm\,5.49}} & \ensuremath{\mathbf{59.53}\,{\scriptstyle\pm\,1.51}} \\
 & $r=0.40$ & \ensuremath{42.00\,{\scriptstyle\pm\,8.31}} & \ensuremath{34.33\,{\scriptstyle\pm\,9.00}} & \ensuremath{29.53\,{\scriptstyle\pm\,5.74}} & \ensuremath{\mathbf{50.20}\,{\scriptstyle\pm\,2.99}} \\
\midrule
Quantization & INT8 & \ensuremath{56.67\,{\scriptstyle\pm\,4.67}} & \ensuremath{50.47\,{\scriptstyle\pm\,3.88}} & \ensuremath{52.27\,{\scriptstyle\pm\,4.05}} & \ensuremath{\mathbf{60.13}\,{\scriptstyle\pm\,0.61}} \\
 & INT6 & \ensuremath{45.20\,{\scriptstyle\pm\,9.18}} & \ensuremath{39.80\,{\scriptstyle\pm\,9.90}} & \ensuremath{34.93\,{\scriptstyle\pm\,4.54}} & \ensuremath{\mathbf{55.67}\,{\scriptstyle\pm\,1.60}} \\
 & INT4 & \ensuremath{9.60\,{\scriptstyle\pm\,2.31}} & \ensuremath{\mathbf{9.87}\,{\scriptstyle\pm\,5.46}} & \ensuremath{2.20\,{\scriptstyle\pm\,0.35}} & \ensuremath{8.20\,{\scriptstyle\pm\,2.42}} \\
\bottomrule
\end{tabular}%
}
\end{table*}

\begingroup
\makeatletter
\setlength{\@fptop}{0pt}
\makeatother
\paragraph{Ablation and hyperparameter results.}
We report WebShop ablations and hyperparameter sweeps below. Vanilla ($\rho=0$) and full SPrPO ($\rho=0.2$, $\alpha=1$) reuse results from Table~\ref{tab:main-robustness}; other configurations use separate ablation evaluation batches.

\begin{table*}[htbp]
\centering
\small
\setlength{\tabcolsep}{4pt}
\renewcommand{\arraystretch}{1.10}
\caption{Hyperparameter study. Success rate (\%).}
\label{tab:more-hyperparameters-success}
\resizebox{\linewidth}{!}{%
\begin{tabular}{@{}lcccccc@{}}
\toprule
Value & Clean & \shortstack{Gaussian\\$\sigma=.20$} & \shortstack{Dropout\\$p=.05$} & \shortstack{Struct.\\$r=.20$} & \shortstack{Unstruct.\\$r=.40$} & INT6 \\
\midrule
\multicolumn{7}{l}{\textit{$\rho$ sweep ($\alpha=1$)}} \\
\midrule
0 & \ensuremath{57.00\,{\scriptstyle\pm\,4.39}} & \ensuremath{18.73\,{\scriptstyle\pm\,3.03}} & \ensuremath{12.20\,{\scriptstyle\pm\,3.30}} & \ensuremath{2.27\,{\scriptstyle\pm\,1.86}} & \ensuremath{42.00\,{\scriptstyle\pm\,8.31}} & \ensuremath{45.20\,{\scriptstyle\pm\,9.18}} \\
0.1 & \ensuremath{55.33\,{\scriptstyle\pm\,5.16}} & \ensuremath{17.67\,{\scriptstyle\pm\,5.52}} & \ensuremath{10.73\,{\scriptstyle\pm\,2.93}} & \ensuremath{4.27\,{\scriptstyle\pm\,4.45}} & \ensuremath{40.80\,{\scriptstyle\pm\,9.71}} & \ensuremath{39.53\,{\scriptstyle\pm\,9.98}} \\
0.2 & \ensuremath{\mathbf{60.00}\,{\scriptstyle\pm\,0.72}} & \ensuremath{\mathbf{30.13}\,{\scriptstyle\pm\,1.81}} & \ensuremath{\mathbf{20.87}\,{\scriptstyle\pm\,2.00}} & \ensuremath{\mathbf{12.40}\,{\scriptstyle\pm\,3.65}} & \ensuremath{\mathbf{50.20}\,{\scriptstyle\pm\,2.99}} & \ensuremath{\mathbf{55.67}\,{\scriptstyle\pm\,1.60}} \\
0.3 & \ensuremath{48.27\,{\scriptstyle\pm\,2.50}} & \ensuremath{25.93\,{\scriptstyle\pm\,2.02}} & \ensuremath{17.00\,{\scriptstyle\pm\,1.97}} & \ensuremath{7.60\,{\scriptstyle\pm\,2.62}} & \ensuremath{40.47\,{\scriptstyle\pm\,3.56}} & \ensuremath{39.27\,{\scriptstyle\pm\,4.75}} \\
\midrule
\multicolumn{7}{l}{\textit{$\alpha$ sweep ($\rho=0.2$)}} \\
\midrule
0 & \ensuremath{48.80\,{\scriptstyle\pm\,0.80}} & \ensuremath{27.07\,{\scriptstyle\pm\,3.72}} & \ensuremath{18.93\,{\scriptstyle\pm\,3.42}} & \ensuremath{11.07\,{\scriptstyle\pm\,6.92}} & \ensuremath{47.07\,{\scriptstyle\pm\,3.11}} & \ensuremath{48.27\,{\scriptstyle\pm\,3.06}} \\
0.5 & \ensuremath{52.87\,{\scriptstyle\pm\,0.61}} & \ensuremath{24.40\,{\scriptstyle\pm\,5.10}} & \ensuremath{16.00\,{\scriptstyle\pm\,5.39}} & \ensuremath{7.07\,{\scriptstyle\pm\,4.47}} & \ensuremath{38.00\,{\scriptstyle\pm\,8.01}} & \ensuremath{47.87\,{\scriptstyle\pm\,3.72}} \\
1 & \ensuremath{\mathbf{60.00}\,{\scriptstyle\pm\,0.72}} & \ensuremath{\mathbf{30.13}\,{\scriptstyle\pm\,1.81}} & \ensuremath{\mathbf{20.87}\,{\scriptstyle\pm\,2.00}} & \ensuremath{\mathbf{12.40}\,{\scriptstyle\pm\,3.65}} & \ensuremath{\mathbf{50.20}\,{\scriptstyle\pm\,2.99}} & \ensuremath{\mathbf{55.67}\,{\scriptstyle\pm\,1.60}} \\
2 & \ensuremath{57.40\,{\scriptstyle\pm\,2.09}} & \ensuremath{26.80\,{\scriptstyle\pm\,4.70}} & \ensuremath{19.87\,{\scriptstyle\pm\,4.36}} & \ensuremath{8.93\,{\scriptstyle\pm\,6.85}} & \ensuremath{46.60\,{\scriptstyle\pm\,4.85}} & \ensuremath{52.40\,{\scriptstyle\pm\,6.13}} \\
\bottomrule
\end{tabular}%
}
\end{table*}

\clearpage

\begin{table*}[!t]
\centering
\small
\setlength{\tabcolsep}{4pt}
\renewcommand{\arraystretch}{1.10}
\caption{Component ablations on WebShop. Success rate (\%).}
\label{tab:more-components}
\resizebox{\linewidth}{!}{%
\begin{tabular}{@{}lcccccc@{}}
\toprule
Training variant & Clean & \shortstack{Gaussian\\$\sigma=.20$} & \shortstack{Dropout\\$p=.05$} & \shortstack{Struct.\\$r=.20$} & \shortstack{Unstruct.\\$r=.40$} & INT6 \\
\midrule
w/o perturbation ($\rho=0$) & \ensuremath{57.00\,{\scriptstyle\pm\,4.39}} & \ensuremath{18.73\,{\scriptstyle\pm\,3.03}} & \ensuremath{12.20\,{\scriptstyle\pm\,3.30}} & \ensuremath{2.27\,{\scriptstyle\pm\,1.86}} & \ensuremath{42.00\,{\scriptstyle\pm\,8.31}} & \ensuremath{45.20\,{\scriptstyle\pm\,9.18}} \\
w/o sensitivity ($\alpha=0$) & \ensuremath{48.80\,{\scriptstyle\pm\,0.80}} & \ensuremath{27.07\,{\scriptstyle\pm\,3.72}} & \ensuremath{18.93\,{\scriptstyle\pm\,3.42}} & \ensuremath{11.07\,{\scriptstyle\pm\,6.92}} & \ensuremath{47.07\,{\scriptstyle\pm\,3.11}} & \ensuremath{48.27\,{\scriptstyle\pm\,3.06}} \\
w/o scheduling & \ensuremath{52.13\,{\scriptstyle\pm\,0.42}} & \ensuremath{26.40\,{\scriptstyle\pm\,4.70}} & \ensuremath{19.33\,{\scriptstyle\pm\,5.32}} & \ensuremath{3.60\,{\scriptstyle\pm\,2.88}} & \ensuremath{45.00\,{\scriptstyle\pm\,4.20}} & \ensuremath{49.73\,{\scriptstyle\pm\,2.42}} \\
\textbf{Full SPrPO} & \ensuremath{\mathbf{60.00}\,{\scriptstyle\pm\,0.72}} & \ensuremath{\mathbf{30.13}\,{\scriptstyle\pm\,1.81}} & \ensuremath{\mathbf{20.87}\,{\scriptstyle\pm\,2.00}} & \ensuremath{\mathbf{12.40}\,{\scriptstyle\pm\,3.65}} & \ensuremath{\mathbf{50.20}\,{\scriptstyle\pm\,2.99}} & \ensuremath{\mathbf{55.67}\,{\scriptstyle\pm\,1.60}} \\
\bottomrule
\end{tabular}%
}
\end{table*}
\clearpage
\endgroup

%% file: appendix/text_examples.tex
\section{Qualitative Examples under Policy Perturbations}
\label{app:text_examples}

In this section, we provide qualitative examples comparing Vanilla and SPrPO on ALFWorld and WebShop under different policy perturbations.

\begingroup
\providecommand{\TEHeaderColor}{6C655E}
\providecommand{\TEHeaderTextColor}{FFFFFF}
\definecolor{TEheader}{HTML}{\TEHeaderColor}
\definecolor{TEheadertext}{HTML}{\TEHeaderTextColor}
\definecolor{TEink}{HTML}{5E5852}
\definecolor{TEline}{HTML}{C9D1D9}
\definecolor{TEmuted}{HTML}{536170}
\definecolor{TEfail}{HTML}{9C4143}
\definecolor{TEsuccess}{HTML}{216D62}
\newcounter{TEcase}
\newcommand{\TEraw}[1]{{\ttfamily\detokenize{#1}}}
\newcommand{\TEgap}[1]{\par\smallskip{\color{TEmuted}\itshape #1}\par\smallskip}
\newcommand{\TEobs}[1]{\par{\color{TEmuted}\textit{Feedback:} ``#1''}\par}
\newcommand{\TEstep}[2]{\par\smallskip\textbf{Step #1}\par\nopagebreak\TEraw{#2}\par}
\newcommand{\TEnote}[1]{\par\smallskip{\color{TEmuted}#1}\par}
\newcommand{\TEfinding}[1]{%
  \begin{tcolorbox}[colback=TEink!3,colframe=TEline,boxrule=0.4pt,
    arc=1mm,left=6pt,right=6pt,top=5pt,bottom=5pt,before skip=8pt,after skip=0pt]
    \textbf{What this shows.} #1
  \end{tcolorbox}}
\newcommand{\TEcontext}[2]{%
  \begin{tcolorbox}[colback=TEink!3,colframe=TEline,boxrule=0.4pt,
    arc=1mm,left=6pt,right=6pt,top=5pt,bottom=5pt,before skip=7pt,after skip=7pt]
    \textbf{Shared decision context --- #1}\par #2
  \end{tcolorbox}}
\newenvironment{TEcasebox}[3]{%
  \refstepcounter{TEcase}\label{#1}%
  \begin{tcolorbox}[enhanced,breakable=false,colback=white,colframe=TEline,
    colbacktitle=TEheader,coltitle=TEheadertext,boxrule=0.6pt,arc=1.5mm,
    left=8pt,right=8pt,top=7pt,bottom=7pt,
    fonttitle=\bfseries,title={Case \theTEcase\quad #2},
    before skip=9pt,after skip=7pt]
    \fontsize{9}{11}\selectfont\raggedright
    \textcolor{TEmuted}{#3}\par\smallskip
}{\end{tcolorbox}}
\newenvironment{TEpanel}[3]{%
  \begin{minipage}[t]{0.485\linewidth}\vspace{0pt}%
  \begin{tcolorbox}[colback=#1!3,colframe=#1!40,colbacktitle=#1!12,
    coltitle=#1,boxrule=0.4pt,arc=1mm,left=6pt,right=6pt,top=5pt,bottom=6pt,
    fonttitle=\bfseries,title={#2\hfill #3},before skip=0pt,after skip=0pt]
  \fontsize{9}{11}\selectfont\raggedright
}{\end{tcolorbox}\end{minipage}}

\begin{TEcasebox}{case:format}{ALFWorld: malformed action tags}
  {\textbf{Perturbation:} Gaussian noise $\sigma=0.2$ \enspace|\enspace ID}
\textbf{Task.} put two soapbar in garbagecan.\par
\textbf{Without perturbation.} Both succeed: Vanilla 8 steps; SPrPO 9 steps.
\TEcontext{before Step 2}{%
Both first go to countertop 1, where the observation lists
``a soapbar 4, a soapbar 1''. Both unperturbed controls also receive
this identical complete input.}
\noindent
\begin{TEpanel}{TEfail}{Vanilla}{FAIL}
\TEstep{2}{<take soapbar 1 from countertop 1>}
\TEobs{Nothing happens.}
\TEnote{Required action tags are missing.}
\TEstep{3}{<action>take soapbar 1 from countertop 1</action>}
\TEobs{You pick up the soapbar 1 from the countertop 1.}
\TEstep{4}{<move soapbar 1 to garbagecan 1>}
\TEobs{Nothing happens.}
\TEgap{Step 5 omitted: goes to cabinet 1.}
\TEstep{6--7 (same output twice)}{<open cabinet 1>}
\TEstep{8}{<examine cabinet 1>}
\TEstep{9}{<go to garbagecan 1>}
\medskip\textbf{Outcome: failure.}\par
6 of the 9 recorded outputs have malformed tags.
\end{TEpanel}\hfill
\begin{TEpanel}{TEsuccess}{SPrPO}{SUCCESS}
\TEstep{2}{<action>take soapbar 4 from countertop 1</action>}
\TEobs{You pick up the soapbar 4 from the countertop 1.}
\TEgap{Step 3 omitted: goes to garbagecan 1.}
\TEstep{4}{<action>move soapbar 4 to garbagecan 1</action>}
\TEobs{You move the soapbar 4 to the garbagecan 1.}
\TEgap{Steps 5--6 omitted: visits toilet 1, then returns to countertop 1.}
\TEstep{7}{<action>take soapbar 1 from countertop 1</action>}
\TEobs{You pick up the soapbar 1 from the countertop 1.}
\TEgap{Step 8 omitted: goes to garbagecan 1.}
\TEstep{9}{<action>move soapbar 1 to garbagecan 1</action>}
\medskip\textbf{Outcome: success in 9 steps.}\par
All 9 outputs retain valid action tags; both soapbars reach the garbagecan.
\end{TEpanel}
\TEfinding{Vanilla briefly recovers at Step 3 but repeatedly loses the required
format. Its first failed command names an admissible action, accepted in a
parser-only check after restoring the tags. This identifies an interface failure,
without establishing that a tag repair would rescue the whole episode.
SPrPO preserves the format and completes the task.}
\end{TEcasebox}

\clearpage
\begin{TEcasebox}{case:subgoal}{ALFWorld: undoing subgoal progress}
  {\textbf{Perturbation:} Dropout $p=0.01$ \enspace|\enspace OOD}
\textbf{Task.} find two cd and put them in safe.\par
\textbf{Without perturbation.} Both succeed: Vanilla 18 steps; SPrPO 20 steps.
\TEcontext{before Step 3}{%
Both visit desk 1 and take cd 1 in Steps 1--2. Current observation:
``You pick up the cd 1 from the desk 1.''
The task still requires placing \emph{both} CDs in the safe.}
\noindent
\begin{TEpanel}{TEfail}{Vanilla}{FAIL}
\TEstep{3}{go to shelf 1}
\TEstep{4}{move cd 1 to shelf 1}
\TEobs{You move the cd 1 to the shelf 1.}
\TEgap{Steps 5--8 omitted: takes cd 2 from desk 2, reaches and opens the safe.}
\TEstep{9}{move cd 2 to safe 1}
\TEobs{You move the cd 2 to the safe 1.}
\TEgap{Steps 10--18 omitted: retrieves cd 1, leaves it on shelf 5, returns to the safe.}
\TEstep{19}{take cd 2 from safe 1}
\TEstep{20}{move cd 2 to safe 1}
\TEnote{Retrieves and replaces an already delivered CD.}
\TEgap{Steps 21--37 omitted: further relocations and repeated visits.}
\TEstep{38}{move cd 1 to shelf 4}
\TEobs{You move the cd 1 to the shelf 4.}
\TEgap{Steps 39--49 omitted: removes cd 2 again, leaves it on shelf 6, then retrieves it.}
\TEstep{50}{move cd 2 to safe 1}
\medskip\textbf{Outcome: failure at the 50-step limit.}\par
cd 1 remains on shelf 4; only cd 2 is in the safe.
\end{TEpanel}\hfill
\begin{TEpanel}{TEsuccess}{SPrPO}{SUCCESS}
\TEstep{3}{go to safe 1}
\TEstep{4}{open safe 1}
\TEstep{5}{move cd 1 to safe 1}
\TEobs{You move the cd 1 to the safe 1.}
\TEnote{First delivery is complete.}
\TEstep{6}{go to desk 2}
\TEstep{7}{take cd 2 from desk 2}
\TEobs{You pick up the cd 2 from the desk 2.}
\TEstep{8}{go to safe 1}
\TEnote{Observation before Step 9:}
``The safe 1 is open. In it, you see a cd 1.''
\TEstep{9}{move cd 2 to safe 1}
\medskip\textbf{Outcome: success in 9 steps.}\par
Both CDs are in the safe. The first delivery is preserved while the second is completed.
\end{TEpanel}
\TEfinding{The failure is a sequence of executable but unproductive actions:
Vanilla moves cd 1 among shelves and repeatedly removes cd 2 from the target
container. Temporary storage alone does not establish failure; the later
relocations and final state show that the two-object goal remains incomplete.
SPrPO completes the two deliveries in sequence without undoing the first.}
\end{TEcasebox}

\clearpage
\begin{TEcasebox}{case:pruning}{WebShop: invalid click targets after pruning}
  {\textbf{Perturbation:} Structured pruning, ratio $0.20$ \enspace|\enspace goal 309}
\textbf{Task.} Find me eco friendly, heavy duty, assembly required, easy clean
file cabinets with size: 23.6"w x 15.7"d x 29.3"h, and price lower than 200.00 dollars\par
\textbf{Without perturbation.} Both succeed in 4 steps (task score 1.0).
\TEcontext{before Step 5}{%
After different search queries (Step 1 omitted), both open product
\TEraw{B08MFGNXW3} at Step 2 and select
\TEraw{23.6"w x 15.7"d x 29.3"h} at Steps 3 and 4.
Thus, \emph{both have already repeated the size selection once}.
Page excerpts: ``DEVAISE 2-Drawer Wood Lateral File Cabinet with Lock for Office
Home, Black''; ``Price: \$169.99''; ``Buy Now''.}
\noindent
\begin{TEpanel}{TEfail}{Vanilla}{FAIL}
\TEstep{5}{click[23.6"w x 15.7"d x 29.3"h]}
\TEnote{Selects the same size a third time.}
\TEstep{6}{click[back to search]}
\TEstep{7}{search[eco friendly, heavy duty, assembly required file cabinets]}
\TEnote{A results page lists product IDs and navigation links.}
\TEstep{8}{click[eco friendly, heavy duty, assembly required file cabinets]}
\TEnote{The query text is not a clickable target. The results page does not change.}
\TEgap{Steps 9--10 omitted: returns to search and issues a shorter query.}
\TEstep{11--15 (same action five times)}{click[eco friendly, heavy duty file cabinets]}
\TEnote{Again clicks query text that is absent from the available targets.}
\medskip\textbf{Outcome: failure at the 15-step limit.}\par
No purchase; task score 0.0.
\end{TEpanel}\hfill
\begin{TEpanel}{TEsuccess}{SPrPO}{SUCCESS}
\TEstep{5}{click[Buy Now]}
\medskip\textbf{Outcome: success in 5 steps.}\par
Task score 1.0.
\TEnote{Purchases the already selected product with the requested size.
The episode ends here.}
\medskip\textbf{Decision contrast.}\par
At the identical Step-5 input, SPrPO proceeds to purchase. Vanilla first
repeats the option selection, then abandons the product and repeatedly clicks
search-query text.
\end{TEpanel}
\TEfinding{Vanilla reaches the product page and selects the requested size,
but fails to finish the transaction. Its later commands use a valid
\TEraw{click[...]} structure with targets that the page does not provide.
SPrPO also repeats the size selection once, but then buys.}
\end{TEcasebox}

\clearpage
\begin{TEcasebox}{case:quantization}{WebShop: option repetition without purchase}
  {\textbf{Perturbation:} INT6 weight quantization (RTN) \enspace|\enspace goal 295}
\textbf{Task.} Find me loose fit, day comfort, hand wash women's tops, tees \&
blouses with short sleeve, polyester spandex for teen girls with color: a6-black,
and size: xx-large, and price lower than 40.00 dollars\par
\textbf{Without perturbation.} Both succeed in 5 steps (task score 1.0).
\TEcontext{before Step 5}{%
After their search queries (Step 1 omitted), both execute the same action text
in Steps 2--4:\par
\TEraw{click[B09NPML43M]}\par
\TEraw{click[a6-black]}\par
\TEraw{click[xx-large]}\par
Both requested options have been selected. The product page still offers
``xx-large'' and ``Buy Now'', and lists ``Price: \$1.01 to \$1.74''.}
\noindent
\begin{TEpanel}{TEfail}{Vanilla}{FAIL}
\TEstep{5}{click[xx-large]}
\TEnote{Repeats the selected size instead of buying.}
\TEstep{6--15 (same action ten times)}{click[xx-large]}
\TEnote{The product-page observation remains unchanged throughout this loop.}
\medskip\textbf{Outcome: failure at the 15-step limit.}\par
12 consecutive size clicks in total (Steps 4--15); no purchase; task score 0.0.
\end{TEpanel}\hfill
\begin{TEpanel}{TEsuccess}{SPrPO}{SUCCESS}
\TEstep{5}{click[Buy Now]}
\medskip\textbf{Outcome: success in 5 steps.}\par
Task score 1.0.
\TEnote{Purchases after selecting the requested color and size.
The episode ends here.}
\medskip\textbf{Decision contrast.}\par
The same Step-5 input leads to another size click for Vanilla and the
terminal purchase action for SPrPO.
\end{TEpanel}
\TEfinding{Vanilla reaches the product and selects both requested options,
but then repeatedly clicks the already selected size until the step limit.
Its actions remain well formed and refer to an option on the page, yet never
complete the transaction. SPrPO proceeds to purchase from the same decision input.}
\end{TEcasebox}

\endgroup